\documentclass{article} 
\usepackage{nips15submit_e,times}

\usepackage{hyperref}
\usepackage{url}
\usepackage[utf8]{inputenc}
\usepackage{amsmath,amssymb}
\newcommand{\vertiii}[1]{{\left\vert\kern-0.25ex\left\vert\kern-0.25ex\left\vert #1 
    \right\vert\kern-0.25ex\right\vert\kern-0.25ex\right\vert}}
\usepackage{bbm}
\usepackage{graphicx}
\usepackage{subfig}
\usepackage{caption}
\usepackage{algorithm}
\usepackage{algorithmic}
\usepackage{epsfig}
\usepackage{amsthm}
\usepackage{epstopdf}
\usepackage{tabularx}
\usepackage{wrapfig}
\usepackage{enumitem}
\usepackage{mathrsfs}
\newtheorem{lemma}{Lemma}
\newtheorem{theorem}{Theorem}
\newtheorem{remark}{Remark}
\newtheorem{corollary}{Corollary}
\newtheorem{proposition}{Proposition}
\newtheorem{example}{Example}

\newtheorem{assumption}{Assumption}
\newtheorem{definition}{Definition}

\DeclareMathOperator*{\argmin}{\arg\!\min}
\DeclareMathOperator*{\argmax}{\arg\!\max}

\usepackage{amssymb,amsmath}
\input xy
\xyoption{all}

\title{Efficient Neighborhood Selection for Gaussian Graphical Models}
\author{Yingxiang Yang \\Dept. of ECE \\ University of Illinois, Urbana-Champaign \\\texttt{yyang172@illinois.edu}\and Jalal Etesami \\Dept. of ISE \\University of Illinois, Urbana-Champaign\\\texttt{etesami@illinois.edu} \and Negar Kiyavash \\Dept. of ISE \\University of Illinois, Urbana-Champaign \\\texttt{kiyavash@illinois.edu} }
\nipsfinalcopy
\begin{document}

\maketitle
\begin{abstract}
This paper addresses the problem of neighborhood selection for Gaussian graphical models. We present two heuristic algorithms: a forward-backward greedy algorithm for general Gaussian graphical models based on mutual information test, and a threshold-based algorithm for walk summable Gaussian graphical models. Both algorithms are shown to be structurally consistent, and efficient. Numerical results show that both algorithms work very well.
\end{abstract}
\section{Introduction}

Gaussian graphical model is a very powerful statistical tool in many fields. By encoding the conditional dependency structure of different variables into the structure of a sparse graph, it helps reveal simple structure beneath high dimensional, and often complicated observed data. For example, in bioinformatics, Gaussian graphical models are often used to represent the Markovian dependence structure among a vast pool of genes, based on observations on gene expression data. In other fields, Guassian graphical models can be used for spatial interpolation. On a graph where nodes represents the geological locations and edges represents direct impact, Gaussian graphical models can be used to infer data at unobserved locations when data for a small number of locations are available. In geostatistics, Gaussian graphical models with small neighborhood size can be used to approximate much denser Gaussian random fields, which greatly reduces the computational complexity due to the availability of fast computation techniques for sparse matrices \cite{rue2005gaussian}. Meanwhile, Gaussian graphical models also play an important role in belief propagation, and the related applications such as error control coding \cite{malioutov2006walk}.

For many applications of Gaussian graphical models, estimating the graphical structure, whose adjacency matrix is essentially the support of the inverse covariance matrix of the joint Gaussian distribution, has been a very interesting topic. Generally speaking, most effort approaches the problem from two different aspect, either to estimate the inverse covariance matrix as a whole, or to combine the estimation of the entries of the neighborhood of each individual node. Two typical examples are Graphical Lasso (GLasso) \cite{friedman2008sparse}, and Neighborhood Lasso (NLasso)\cite{meinshausen2006high}, respectively.

\subsection{Related Work}

Recently, there has been several more efficient methods that recover the structure of the neighborhood (or the entire graphical structure) greedily. In \cite{johnson2011high}, Johnson et.al proposed a forward-backward greedy method in estimating the neighborhood entries of each node. The forward part of the method greedily selects the node that essentially maximizes the mutual information between the node of interest and the estimated neighborhood. Since no theoretical guarantees can be provided that such selection would always be correct unless the size of the neighborhood is 1, the backward pruning algorithm is designed to prune potentially false neighbors. Theoretically, the authors proved that under certain conditions, with the most important one being the restricted eigenvalue condition, the forward-backward algorithm is structurally consistent, with a sample complexity at the state of the art level, superior to NLasso. Another interesting work is \cite{anandkumar2012high}, in which the authors showed that if the Gaussian graphical model is walk summable, and the number of paths whose lengths are at most $\gamma$ between a node and any of its non-neighbors is restricted to be smaller than a fixed number $\eta$, then the conditional mutual information between a node and a non-neighbor is upper bounded, while the conditional mutual information between a node and a neighbor can be lower bounded, when the set conditioned on includes at least one node for each distinct path of at most length $\gamma$ between the node and its non-neighbor. Hence, computing all the conditional mutual information between two nodes where the set conditioned on has at most cardinality $\eta$, and applying a threshold would succeed in finding the neighborhood, asympotically. Finally, the authors showed that under further assumptions, the test is structurally consistent for almost every graph as the graph size approaches infinity.

Even more recently, there has been a line of research that focuses on lower bounding the maximal influence between a node and its undiscovered neighbors, or certain information distance between two graphs if they differ by at least one edge. In \cite{bresler2015efficiently}, Bresler showed that for Ising models, if the neighborhood hasn't been completely discovered, then the maximal influence for a node from one of its undiscovered neighbors can be lower bounded by a constant away from 0. Hence, thresholding the influence repeatedly would guarantee the selection of all neighbors. Moreover, the total number of nodes selected can be upper bounded by a quantity independent of the graph size. Hence pruning non-neighbors can be done efficiently. For Gaussian graphical models, Jog et.al showed in \cite{DBLP:journals/corr/JogL15} that if two Gaussian graphical models differ by at least one edge, then the KL divergence between the joint distributions of the two models is bounded away from 0. Hence, if we are given a set of sparse candidate graphs, one can simply obtain the solution by maximizing the likelihood.

\subsection{Our contribution}

In this work, we follow the footsteps of the research effort mentioned above. In particular, we focus on two questions. (1) Can we greedily select the neighborhood of each node based on an information theoretic measurement? (2) Does there exist a natural way of measuring the influence between two nodes, and can the maximum influence between a node and its undiscovered neighbors be bounded away from zero at any time?

We give affirmative answers to both questions. For the first question, we develop a forward-backward type greedy algorithm, which selects the neighborhood of each node iteratively. In each round, the algorithm picks a new neighbor that maximize the conditional mutual information between two nodes conditioned on the already selected pseudo neighborhood, and prunes all the least likely neighbors, until the conditional mutual information between the node and potential new neighbors is below a threshold. For this algorithm, we show its structural consistency, while the efficiency is demonstrated numerically. For the second question, we show that, for walk summable Gaussian graphical models, one can always lower bound the absolute value of the maximal conditional covariance between a node and its undiscovered neighbor by a constant. This property enables us to design efficient thresholding and pruning algorithms, which first selects a pseudo neighborhood that contains all true neighbors with high probability, and then prunes non-neighbors efficiently. Both the performance and efficiency of the algorithm are characterized theoretically and numerically.

The rest of the paper is organized as follows. In Section \ref{sec:preliminary}, we introduce the general set of notation we adopt, and introduce the preliminaries to Gaussian graphical models. The category of walk summable Gaussian graphical models is also introduced, with more details included in Section \ref{sec:appendix}. In Section \ref{sec:mit}, we propose the greedy neighborhood selection algorithm, and prove its structural consistency. In Section \ref{sec:main}, we propose the thresholding algorithm, and analyze its efficiency, correctness. The pruning algorithm is also given in Section \ref{sec:main}, and the overall structural consistency and sample complexity are derived. In addition, we also show in Section \ref{sec:main} that our assumptions required for the algorithm are not restrictive in the sense that they do not prohibit the graph to scale up to size infinity. In Section \ref{sec:numerical} we demonstrate the performance of our algorithm, and compare it to benchmark algorithms. The conclusions are then presented, and the proofs as well as the detailed introduction on walk summability is included in Section \ref{sec:appendix}.

\section{Preliminaries}
\label{sec:preliminary}

\subsection{Gaussian graphical models}

Throughout the paper, the following set of notation is commonly used. We denote the underlying Gaussian graphical model by $(V,E)$, where $V=\{1,...,n\}$ with index $i$ corresponding to the $i$-th dimension of the joint distribution. The covariance and inverse covariance matrices are denoted by $\Sigma$ and $J$, respectively. The empirical covariance matrix is denoted by $\hat{\Sigma}$. For either $\Sigma$ or $J$, we use $\Sigma_{S,T}$ or $J_{S,T}$ to denote the submatrices obtained by picking the intersections of the rows and columns with indices in sets $S$ and $T$, respectively. We assume that $S$ and $T$ are ordered sets, i.e., if $S=\{s(1),...,s(|S|)\}$ contains more than 1 element, then $s(i)<s(j)$ for $i<j$. This rigorously specifies the way of writing down the submatrix, and permits us to refer to the $i$-th element of $S$, which corresponds to a unique node in the graph, without raising any confusions. When referring to an element instead of a submatrix, we simply write $\Sigma_{ij}$ or $J_{ij}$ where $i$ and $j$ are the indices of the nodes involved. For a given graph, we denote the neighbors of a node $i$ by $\mathcal{N}_i$, and let $\bar{\mathcal{N}}_i$ be all the non-neighbors of $i$. The estimated set of neighbors for node $i$ is denoted as $S_i$. In addition, for any set $S$, we denote its complement by $S^c$.

Given an $n$-dimensional Gaussian random vector $X=(X_1,...,X_n)\in\mathbbm{R}^n$, there are two common ways to write the probability density function $\mu(x)$. The first way is to write $\mu(x)$ in the covariance form, characterized by the mean vector $m=\mathbbm{E}[X]$ and the covariance matrix $\Sigma=\mathbbm{E}[(X-m)^2]$. The second way is to write $\mu(X)$ in the information form, characterized by the information matrix (also known as the precision matrix or the inverse covariance matrix) $J=\Sigma^{-1}$ and the potential vector $h=\Sigma^{-1}m=Jm$. More specifically,
\begin{align}
\mu(X)\propto \exp\left\{-\frac{1}{2}(X-m)^T\Sigma^{-1}(X-m)\right\}\propto\exp\left\{-\frac{1}{2}X^TJX+h^TX\right\}.
\end{align}
The information matrix $J\in\mathbbm{R}^{n\times n}$ is symmetric and positive definite. It encodes the conditional dependency structure between different dimensions of $X$. The most well known result is that for any $i,j\in V$, $J_{ij}=0$ if and only if $X_i\perp X_j|X_{V\backslash\{i,j\}}$.

Furthermore, for any $S\subset V$, denote $S^c:=V\backslash S$, then it is known that
\begin{align}
\label{eq:conditioning}
J_{S,S}&=\Sigma^{-1}_{S,S|S^c}=(\Sigma_{S,S}-\Sigma_{S,S^c}\Sigma_{S^c,S^c}^{-1}\Sigma_{S,S^c}^T)^{-1},\\\label{eq:marginalization}
\Sigma_{S,S}&=J_{S,S|S^c}^{-1}=(J_{S,S}-J_{S,S^c}J_{S^c,S^c}^{-1}J_{S,S^c}^T)^{-1}.
\end{align}
These two equations present two common operations used for Gaussian graphical models: conditioning, and marginalization. When marginalization is performed (equation (\ref{eq:marginalization})), a new set of edges is introduced into the subgraph with vertexes $S$ by the additive term $J_{S,S^c}J_{S^c,S^c}^{-1}J_{S,S^c}^T$. If we wish to preserve the original graphical structure over vertexes $S$, then the operation needed is conditioning, as indicated in equation (\ref{eq:conditioning}).

When building the thresholding algorithm, we consider a category of Gaussian graphical models with strong intuition, the walk summable Gaussian graphical models. For the purpose of analysis, it is sufficient to know that such category of Gaussian graphical models is parameterized by $\alpha\in(0,1)$, and a Gaussian graphical model is said to be $\alpha$ walk summable if $\Vert|I-\sqrt{D}^{-1}J\sqrt{D}^{-1}|\Vert_2\leq \alpha$, where $D$ is the diagonal matrix of $J$, and $\Vert\cdot\Vert_2$ is the spectral norm. The strong intuition of $\alpha$ walk summable Gaussian graphical models is that the covariance $\Sigma_{ij}$ can be represented as the summation of walk weights along the edges of the graph from $i$ to $j$. A detailed introduction on walk summable Gaussian graphical models can be found in Appendix \ref{sec:WSGMRF}.

\subsection{Information theoretic quantities}
The mutual information between $X_i$ and $X_j$ and the conditional mutual information between $X_i$ and $X_j$ conditioned on a set of random variables $X_S$ are defined as
\begin{align}\nonumber
I(X_i;X_j)=\mathbb{E}\left[\log\frac{f_{X_i,X_j}}{f_{X_i}f_{X_j}}\right],\ \ \
I(X_i;X_j|X_S)=\mathbb{E}\left[\log\frac{f_{X_i,X_j|X_S}}{f_{X_i|X_S}f_{X_j|X_S}}\right],
\end{align}
where $f$ denotes the probability density function and the expectations are with respect to the joint distributions. For a set of jointly Gaussian random variables, it is known that \cite{cover2012elements}
\begin{align}
I(X_i;X_j)=\frac{1}{2}\log\frac{1}{1-\rho^2_{i,j}},\ \ \  I(X_i;X_j|X_S)=\frac{1}{2}\log\frac{1}{1-\rho^2_{i,j|S}},
\end{align}
where $\rho_{i,j}$ and $\rho_{i,j|S}$ are the correlation coefficient between $X_i$ and $X_j$ and the conditional correlation coefficient given $X_{S}$, respectively. Recall the definition of the conditional correlation coefficient:
$\rho_{i,j|S}=\Sigma_{ij|S}/\sqrt{\Sigma_{ii|S}\Sigma_{jj|S}}.$
Hence, the empirical mutual information is given by $\hat{I}(X_i;X_j|X_S):=\frac{1}{2}\log\frac{\hat{\Sigma}_{ii|S} \hat{\Sigma}_{jj|S}}{\hat{\Sigma}_{ii|S}\hat{\Sigma}_{jj|S}-\hat{\Sigma}^2_{ij|S}}$.

\subsection{Forward-Backward algorithm}

Zhang introduced a forward-backward greedy algorithm for sparse linear regression that begins with an empty set of active variables and gradually adds and removes variables to the active set \cite{zhang2009adaptive}. The algorithm has two steps: the forward step and the backward step. In the forward step, the algorithm finds the \textit{best} next candidate that minimizes a loss function and adds it to the active set as long as the improvement is greater than a certain threshold. In the backward step, the algorithm checks the influence of those variables that are already collected and if the contribution of some variables in reducing the loss function is less than a certain threshold, the algorithm removes them from the active set. The algorithm is required to repeat the backward step until no such nodes are found. By choosing an appropriate threshold, the algorithm terminates in a finite number of steps. Jalali et al. and Johnson et al. utilized this method in conjunction with a Gaussian log-likelihood function to learn the underlying structure of spare GMRFs and proved the consistency of their algorithm for sparse Gaussian models \cite{jalali2011learning, johnson2011high}. 

We also adopt the forward backward method in one of our learning algorithms. The algorithm at the forward step adds the best candidate to the active set of a node of interest based on a conditional mutual information test. In the backward step, the algorithm removes all the unlikely neighbors from the active set at one shot unlike the aforementioned related work, which do so at intermediate steps. The algorithm repeats the two steps to estimate the neighborhoods of all the nodes in the graph one node at a time.

Next section presents the motivations behind our two-step approach using the geometric interpretation of the conditional mutual information test.

\subsection{Geometric representation}\label{sec:geo}

For zero-mean Gaussian random variables, the conditional mutual information between $X_i$ and $X_j$ given $X_S$ is related to the minimum distance between the rejection vectors of $X_i$ and $X_j$ from the subspace spanned by $X_S$ in the Hilbert space of second-order random variables. As depicted in Figure \ref{fig:geometry}, let $Y_i$ and $Y_j$ be the rejection vectors of $X_i$ and $X_j$ from the subspace $[X_S]$ spanned by $X_S$, respectively. Then the minimum distance between $Y_i$ and $Y_j$ is related to the conditional mutual information $I(X_i;X_j|X_S)$. We establish this relationship in the next Lemma.
\begin{figure}[tb]
\hspace{4.7cm}
\includegraphics[width=0.31\linewidth]{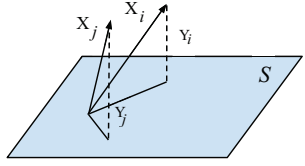}
\caption{Geometric representation of $I( X_i; X_j|X_S)$.}
\label{fig:geometry}
\end{figure}
\begin{lemma}
\label{lm:relationship}
Let $\{X_i,X_j,X_S\}$ be a subset of a zero-mean multivariate Gaussian vector $X=(X_1,...,X_n)^T$. Then
\begin{align}
2I(X_i;X_j|X_S)=\log\mathbb{E}[Y_i^2]-\log\min_{\alpha}\mathbb{E}[(Y_i-\alpha Y_j)^2],
\end{align}
where $Y_i=X_i-(\beta')^TX_S$ with $\beta'=\argmin_{\beta}\mathbb{E}[(X_i-\beta^T X_S)^2]$, and  $Y_j=X_j-(\beta'')^TX_S$ with $\beta''=\argmin_{\beta}\mathbb{E}[(X_j-\beta^T X_S)^2]$.
\end{lemma}
\begin{proof}
Proof See Appendix \ref{lm:relationship_proof}.
\end{proof}

 \begin{corollary}\label{coro}
 Let $X$ be an $n$-dimensional zero-mean Gaussian vector with corresponding GMRF, $G=(V,E)$. Then, for every node $i\in V$, $Y_i$, the rejection vector of $X_i$ from $[X_{\mathcal{N}_i}]$, is orthogonal to $X_j$ for every $t\notin\{i\}\cup \mathcal{N}_i$.
 \end{corollary}
\begin{proof}
From the definition of GMRF, we have $I(X_{i};X_j|X_{\mathcal{N}_i})=0$ for every $t\notin\{i\}\cup \mathcal{N}_i$. Lemma \ref{lm:relationship} implies $\mathbb{E}[Y_i Y_j]=0$. Moreover, $\mathbb{E}[Y_i X'_j]=0$. This follows from the definition of $X'_j$ and because  $Y_i$ is orthogonal to $[X_{\mathcal{N}_i}]$. Hence, $\mathbb{E}[Y_i X_j]=\mathbb{E}[Y_i Y_j]+\mathbb{E}[Y_i X'_j]=0$.
\end{proof}

 Next result gives a test to identify non-neighbors of a node $i$ from a subset $S$ that does not contain $i$ but contains all neighbors of $i$. This means $\mathcal{N}_i\subseteq S$.
\begin{theorem}\label{prune}
 Let $S\subseteq\{1,...,n\}$ contains all the neighbors of a node $i$, but not $i$. Then, the zero entries of $\Sigma_{i,S}\Sigma_{S,S}^{-1}\sqrt{D_S}$ correspond to the non-neighbors of node $i$, where $D_S$ is a diagonal matrix whose entries are the main diagonal of $\Sigma_{S,S}$.
 \end{theorem}
 \begin{proof}
 See Appendix \ref{proof_prune}.
 \end{proof}

\section{Greedy neighborhood selection via mutual information test}\label{sec:mit}

In this section, we present our first learning algorithm, which outputs $S_i$, the estimated neighborhood of a given node $i$. Initially, $S_i$ is an empty set. At round $k$, the algorithm finds a node that maximizes the conditional mutual information with $i$ conditioned on $S_i^{(k-1)}$, where the superscript is used to denote the number of rounds. Namely, $j_1=\argmax_{j\in (S_i^{(k-1)})^c}\hat{I}(X_i;X_j|X_{S_i^{(k-1)}})$. If the corresponding conditional mutual information is below a given threshold $\epsilon_F$, the algorithm stops and outputs the current estimate. Otherwise, it prunes the unlikely neighbors in set $S_i$ using Theorem \ref{prune}. More precisely, it computes the vector $u^*=\small{\hat{\Sigma}_{i,S_i^{(k-1)}}\hat{\Sigma}_{S_i^{(k-1)},S_i^{(k-1)}}^{-1}\sqrt{\hat{D}_{S_i^{(k-1)}}}}$, and removes those nodes from the active set whose corresponding values in this vector are smaller than a calibrated threshold $\epsilon_B$. Algorithm \ref{algorithm} summarizes these steps. 

\begin{algorithm}[tb]
  \caption{Finding Neighbors of Node $i$}
  \label{algorithm}
    \begin{algorithmic}[1]
       \STATE $Input:$\ \ $\hat{\Sigma}$, $i$, Threshold  $\epsilon_F$, $0<\nu<1$
    \STATE $Output:$\ \ $S_i$
    \STATE { $S_i^{(0)}\leftarrow\emptyset$, and $k\leftarrow 1$.}
    \WHILE {true}  
    \STATE {Next candidate: \small{$j_1\leftarrow\argmax_{j\in (S_i^{(k-1)})^c}\hat{I}(X_i;X_j|X_{S_i^{(k-1)}})$}}
    \STATE {Updating the active set: \small{$S_i^{(k)}\leftarrow S_i^{(k-1)}\cup\{j_1\}$}}
    \STATE {Calibration factor for pruning threshold: \small{$k_{i,j_1}\leftarrow \hat{\Sigma}_{ii}e^{-2\big(\hat{I}(X_i;X_{S_i^{(k-1)}})+\hat{I}(X_{j_1};X_{S_i^{(k-1)}})\big)}$}}
    \STATE{\small{ $k\leftarrow k+1$}}  
    \IF {\small{$\delta:=\hat{I}(X_i;X_{j_1}|X_{S_i^{(k-1)}})<\epsilon_F$}}
    \STATE {break}
    \ENDIF
    \STATE{\small{$u^*\leftarrow \hat{\Sigma}_{i,S_i^{(k-1)}}\hat{\Sigma}_{S_i^{(k-1)},S_i^{(k-1)}}^{-1}\sqrt{\hat{D}_{S_i^{(k-1)}}}$}}
    \STATE{\small{$\epsilon_{B}\leftarrow \sqrt{\nu(1-e^{-2\delta})k_{i,j_1}}$}}
    \STATE{$L\leftarrow\{i:\ |u^{*}_{i}|<\epsilon_{B} \}$}
    \STATE{Pruning step: \small{$S_i^{(k-1)}\leftarrow S_i^{(k-1)}\backslash L$}}
    \ENDWHILE
   \end{algorithmic}
\end{algorithm}

\begin{proposition}\label{one}
If node $i$ has only one neighbor, Algorithm \ref{algorithm} always returns the correct neighbor after one round.
\end{proposition}
\begin{proof}
See Appendix \ref{one_proof}.
\end{proof}

When the neighborhood of a node contains more than one node, Algorithm \ref{algorithm} does not necessary find a neighbor at each iteration. In spite of that, we can prove the structural consistency of the algorithm for sparse GMRFs.

 \subsection{Structural Consistency for Sparse Gaussian Models}
In this section, we prove the structural consistency of Algorithm \ref{algorithm} for a class of Gaussian models that satisfy the so-called restricted eigenvalue property. 
\begin{assumption}\label{assum}
Let $-\{i\}:=\{1,...,n\}\setminus\{i\}$. We assume that $\exists\ C_{min}>0$ and $\rho\geq 1$ such that the partial covariance matrix $\Sigma_{-\{ i\}}:=\mathbb{E}[X_{-\{i\}}X_{-\{i\}}^T]$ satisfies
\vspace{-1mm}
\begin{align}\nonumber
C_{min}\Vert \Delta\Vert_{F}\leq\Vert \Sigma_{-\{ i\}}\Delta\Vert_{F}\leq\rho C_{min}\Vert\Delta\Vert_{F},
\end{align}
where $\Delta$ is an arbitrary sparse vector with at most $\eta d$ non-zero entries, and $\eta\geq 2+4\rho^2(\sqrt{(\rho^2-\rho)/d}+\sqrt{2})^2$. 
\end{assumption}
As it is discussed in \cite{johnson2011high}, restricted eigenvalue assumption imposes a more relaxed condition on the model parameters compared to  the condition imposed by the $\ell_1$-regularized Gaussian MLE \cite{ravikumar2011high} or the condition imposed by the linear neighborhood selection with $\ell_1$-regularization \cite{meinshausen2006high}.  
Furthermore, under the restricted eigenvalue assumption, Johnson et al. \cite{johnson2011high} show the sparsity of the forward-backward greedy algorithm that optimizes the following loss function:
\begin{equation}\label{loss}
\mathcal{L}(\beta)=\mathbb{E}[(X_i-\sum_{j\neq i}\beta_j X_{j})^2].
\end{equation}
This algorithm picks $j^*=\arg\min_{j\in (\text{supp} \beta)^c,\alpha}\mathcal{L}(\beta+\alpha e_j)$ as its best next candidate and removes $\tilde{j}=\arg\min_{j\in\text{supp} \beta}\mathcal{L}(\beta-\beta_j e_j)$ as its least likely neighbor, where $e_j$ is a unite vector with only one non-zero entry, located in the $j$-th position.\\
In order to show the structural consistency of our algorithm under Assumption \ref{assum}, we present the next lemmas that guarantee if the aforementioned forward-backward greedy algorithm with the loss function in (\ref{loss}) returns the neighborhood of a node $i$, so does Algorithm \ref{algorithm}.

\begin{remark}
Unlike the loss function in (\ref{loss}), conditional mutual information criterion selects the next candidate only based on its projection proportion which is geometrically the only proportion that matters for neighborhood selection. 
The reason is as follows: let $j_1$ and $j_2$ be the nodes that are chosen in the 5-th line of Algorithm \ref{algorithm} and using the loss function in (\ref{loss}) given the active set $S_i^{(k-1)}$, respectively. From Lemma \ref{lm:relationship}, we have
\begin{equation}\label{inyeki}\nonumber
\small{j_{1}=\argmax_{j\in(S_i^{(k-1)})^c}\frac{\mathbb{E}^2[Y_{i}Y_{j}]}{\mathbb{E}[Y^2_{j}]},\ \ \ j_{2}=\argmax_{j\in(S_i^{(k-1)})^c}\frac{\mathbb{E}^2[Y_{i}Y_{j}]}{\mathbb{E}[(X'_j)^2]+\mathbb{E}[Y^2_{j}]},}
\end{equation}
where $Y_j$ and $X'_j$ are the rejection and projection components of the orthogonal projection of $X_j$ onto the subspace spanned by $S_i^{(k-1)}$, respectively. 
From Corollary \ref{coro}, we know that the subset $S_i^{(k-1)}$ contains $\mathcal{N}_i$ if and only if $Y_i$ is orthogonal to $Y_j$ for every $t\notin \{i\}\cup S_i^{(k-1)}$.
This implies that only the rejection components of $X_j$ and $X_i$ after the orthogonal projection onto the subspace spanned by $X_{S_i^{(k-1)}}$ are relevant quantities to check whether $S_i^{(k-1)}$ contains the neighborhood of $i$. Hence, geometrically, $j_1$ is a better candidate to be the next neighbor of node $i$.
\end{remark}
\begin{lemma}\label{backi}
Let $\epsilon_{F}=\frac{1}{2}\log\frac{1}{1-\epsilon}$ for some $0<\epsilon<1$, then if Algorithm \ref{algorithm} adds node $j$ to the active set $S$ of node $i$, it decreases the loss function given in (\ref{loss}) by at least 
$k_{i,j}\epsilon$, where $ k_{i,j}= \Sigma_{ii}\exp(-2I(X_i;X_S)-2I(X_j;X_S))>0$.
\end{lemma}
\begin{proof}
 Proof See Appendix \ref{backi_proof}.
 \end{proof}
\begin{lemma}\label{comp}
If the forward greedy algorithm guarantees no false exclusions, then the pruning step excludes all non-neighbors. Moreover, the most unlikely node removed by the forward-backward greedy algorithm always belongs to the set $L$ identified in Algorithm \ref{algorithm}'s pruning step.
\end{lemma}
\begin{proof}
 Proof See Appendix \ref{comp_proof}.
 \end{proof}
 
Under Assumption \ref{assum}, Lemmas 1 and 3 in \cite{jalali2011learning} and Lemma \ref{backi}, guarantee no false exclusions using the mutual information test as long as a proper forward stopping threshold $\epsilon_F$ is selected. Lemma \ref{comp}, guarantees no false inclusions in the backward part of the mutual information test . Hence, we will have the following result:
\begin{theorem}\label{ther}
 In Algorithm \ref{algorithm}, let  $k_i:=\min_{j\neq i}k_{i,j}$, $K_i:=\max_{j\neq i}k_{i,j}$, and $\epsilon_F:=\frac{1}{2}\log\frac{1}{1-\epsilon}$, such that $1>\epsilon>\min\{1-\varepsilon,8c\rho\eta d\log n/(C_{min}Nk_i)\}$, where $d$ is the maximum node degree in the graphical model, $n$ is the number of nodes, $c$ and $0<\varepsilon\ll 1$ are constants. Under Assumption \ref{assum}, if the nonzero entries of vector $|\Sigma_{i,-\{i\}}\Sigma_{-\{i\}}^{-1}|$ are lower bounded by ${\small\sqrt{32\rho\epsilon K_i/C_{min}}}$, and the number of samples  $N>Cd\log n$ for some constant $C$, there exist constants $c_1$ and $c_2$ such that with probability at least $1-c_1\exp(-c_2N)$, Algorithm \ref{algorithm}  will  terminate in finite number of steps and return the exact neighborhood of the given node $i$.
\end{theorem}
\begin{proof}
If $I(X_i;X_j|X_S)>\epsilon_F$, Lemmas \ref{lm:relationship} and  \ref{backi} imply that after adding $t$ to the active set $S$, the loss function $\mathcal{L}$ in  (\ref{loss}) decreases by at least $\epsilon k_i$. On the other hand, from Lemmas \ref{lm:relationship} and \ref{comp}, we know that  nodes that are removed by the pruning step in Algorithm \ref{algorithm} will increase the loss function $\mathcal{L}$ by at most $\nu(1-e^{-2I(X_i;X_j|X_S)})k_{i,j}=\nu\frac{\mathbb{E}^2[Y_i Y_j]}{\mathbb{E}[X^2_j]}$. Note that  $\frac{\mathbb{E}^2[Y_i Y_j]}{\mathbb{E}[X^2_j]}$ is precisely the amount of decrease in $\mathcal{L}$ as a result of adding $j$ to the active set $S$. Thus at each round, the loss function $\mathcal{L}$ reduces by at least $(1-\nu)\epsilon k_i$ and hence, Algorithm \ref{algorithm} terminates within a finite number of steps. 

 Using Lemma 9 in \cite{wainwright2009sharp}, Theorem 2 in \cite{johnson2011high}, and the fact that at each round in Algorithm \ref{algorithm}, $\mathcal{L}$ decreases by at least $\epsilon k_i\geq8c\rho\eta d\log n/C_{min}N$, we obtain that Algorithm \ref{algorithm} given $N$ samples will return the exact neighborhood of $i$ with probability at least $1-c_1\exp(-c_2N)$.
 \end{proof}

\section{Neighborhood selection via thresholding}
\label{sec:main}
In this section, we present Algorithm \ref{alg:thresholding}, which is very similar to the one that was first introduced in \cite{bresler2015efficiently} on Ising models. Unlike the first algorithm we presented, which selects new neighbors and prunes potential false neighbors at the same time, this algorithm selects all potential neighbors first and then prunes false neighbors. For this algorithm, we assume that Assumption \ref{asm:asm} holds and the parameters involved in the inputs of the algorithm, are known.

For a walk summable Gaussian graphical model, the algorithm works with high probability. If we compute the empirical covariance between $i$ and $j$ given set $S_i$ as
\begin{align}
\hat{\Sigma}_{ij|S_i}=\hat{\Sigma}_{ij}-\hat{\Sigma}_{i,S_i}\hat{\Sigma}_{S_i,S_i}^{-1}\hat{\Sigma}_{S_i,j},
\end{align}
where we assume that the number of samples is large enough so that $\hat{\Sigma}_{S_i,S_i}^{-1}$ exists, then the maximum of $\hat{\Sigma}_{ij|S_i}$ can be lower bounded by a constant, where $S_i$ is the estimated neighborhood of node $i$ and the maximium is taken over $j$, the undiscovered neighbors of $i$. Hence, in each round, we can simply select the neighborhood set by thresholding the absolute value of the conditional covariance, which guarantees to find at least one neighbor with high probability. The size of the pseudo neighborhood can be upper bounded, and thus pruning can be performed efficiently.

There are numerous ways of pruning the neighborhood $S_i$ (for example, \cite{johnson2011high}). Here we use an efficient yet simple one, given in Algorithm \ref{alg:pruning}. It simply computes $\Gamma=|\hat{\Sigma}_{i,S_i}\hat{\Sigma}_{S_i,S_i}^{-1}|$, and prunes all the nodes with the corresponding entries that's below a threshold $\nu a$, where $a$ is the same as in Assumption \ref{asm:asm}, and $\nu\in[0,1)$. We prove that this algorithm prunes the neighborhood efficiently, while preserving the neighbors with high probability.

\subsection{The thresholding and pruning algorithms}

\begin{assumption}\label{asm:asm}
Consider a Gaussian graphical model satisfying the following set of assumptions.
\begin{itemize}[nolistsep]
\item The model is $\alpha$ walk summable.
\item The diagonal elements of $J$ is bounded by $d_{min}$ and $d_{max}$.
\item The absolute values of the off-diagonal non-zero elements are lower bounded by $a$, and upper bounded by $b$.
\item The degree of the graph is upper bounded by a known number $\Delta$.
\end{itemize}
\end{assumption}

\begin{remark}
Notice that the second part of the third bullet is not required in order for the algorithms to function correctly. It is only used in deriving theoretical guarantees in a simpler form. In fact, given the first two assumptions, the absolute values of off-diagonal non-zero entries are naturally upper bounded by $d_{min}$ because $J$ is positive definite.
\end{remark}

Under Assumption \ref{asm:asm}, we present Algorithms \ref{alg:thresholding}, \ref{alg:pruning}, in which $\epsilon$ and $\nu$ are set manually. The correctness of these algorithms under the assumption is proven in the next section.

\begin{algorithm}[H]
\caption{Greedy Neighborhood Selection for Node $i$ via Thresholding}
\label{alg:thresholding}
    \begin{algorithmic}[1]
        \STATE {\it Input:} $\hat{\Sigma}, \Delta, \alpha, d_{min}, d_{max}, a, \epsilon$.
        \STATE {\it Output:} $S_i$.
        \STATE { Initialization: $S_i^{(0)}\leftarrow\emptyset$, and $k\leftarrow 0$.}
        \STATE { $\tau\leftarrow ad_{max}^{-1}(d_{max}^2(1+\alpha)-a^2)^{-1}-\epsilon$}
        \WHILE{$k<\Delta$}
            \STATE {For all $j\in V\backslash(\{i\}\cup S_i^{(k)})$, $\hat{\Sigma}_{ij|S_i^{(k)}}\leftarrow \hat{\Sigma}_{ij}-\hat{\Sigma}_{i,S_i^{(k)}}\hat{\Sigma}_{S_i^{(k)},S_i^{(k)}}^{-1}\hat{\Sigma}_{S_i^{(k)},j}$}
            \STATE { $S_i^{(k+1)}\leftarrow S_i^{(k)}\cup\{j\in V\backslash(\{i\}\cup S_i^{(k)}):|\hat{\Sigma}_{ij|S_i^{(k)}}|\geq\tau\}$}
            \STATE {$k\leftarrow k+1$}
            \IF{$S_i^{(k)}=S_i^{(k-1)}$}
                \STATE{break}
            \ENDIF
        \ENDWHILE
        \STATE{$S_i\leftarrow S_i^{(k)}$}
    \end{algorithmic}
\end{algorithm}

\begin{algorithm}[H]
\caption{Pruning the estimated neighborhood $S_i$}
\label{alg:pruning}
    \begin{algorithmic}[1]
        \STATE {\it Input:} $S_i, \hat{\Sigma}, \alpha, d_{min}, d_{max}, a, \nu$.
        \STATE {\it Output:} $S_i^{p}$.
        \STATE { $\tau^p\leftarrow \nu a$}
        \STATE { $\hat{\Gamma}\leftarrow |\hat{\Sigma}_{i,S_i}\hat{\Sigma}_{S_i,S_i}^{-1}|$}
        \STATE { $L\leftarrow\text{Find}(\hat{\Gamma}\leq\tau^p)$}
        \STATE { $S_i^{p}\leftarrow S_i\backslash L$}

    \end{algorithmic}
\end{algorithm}

\subsection{Algorithm Efficiency}

We characterize the efficiency of Algorithm \ref{alg:thresholding} with two upper bounds, which help derive the computational and sample complexity of our algorithms later on: (1) the upper bound for the size of $S_i$, and (2) the upper bound for the number of iterations. These two aspects are good proxies for characterizing the algorithmic efficiency, since a good algorithm should always be able to select all the actual neighbors with a small number of iterations, while keeping the number of non-neighbors in $S_i$ at minimum.

We first upper bound $S_i$ selected by Algorithm \ref{alg:thresholding} with the following result.

\begin{theorem}
\label{thm:SiUB}
Suppose that $\tau$ is the threshold used in Algorithm \ref{alg:thresholding}, Assumption \ref{asm:asm} holds, and assume that the absolute values of off-diagonal non-zero entries of $J$ are upper bounded by $b$. Then the number of nodes selected into $S_i$ is upper bounded by
\begin{align}
|S_i|\leq\frac{b^2}{(1-\alpha)^2d_{min}^2\tau^2}\Delta_i,
\end{align}
where $\Delta_i$ is the actual degree of node $i$.
\end{theorem}
\begin{proof}
See Appendix \ref{thm:SiUBp}.
\end{proof}

The above theorem bounds $|S_i|$ by a constant times $\Delta_i$. However, it is worth pointing out that if the size of the graph is smaller than the upper bound given in Theorem \ref{thm:SiUB}, a tighter upper bound is needed.

Meanwhile, the number of iterations of the algorithm, which is at most $\Delta$, can also be upper bounded.

\begin{proposition}[Upper bound for $\Delta$]\label{prop:relationship}

For a Gaussian graphical model satisfying Assumption \ref{asm:asm}, we must have
\begin{align}
\Delta\leq\left(\frac{d_{max}\alpha}{a}\right)^2.
\end{align}

\end{proposition}
\begin{proof}
See Appendix \ref{prop:relationshipp}.
\end{proof}

\begin{remark}
The upper bound for $\Delta_i$ is independent of the graph size, mainly due to the assumption of $\alpha$ walk summability. By plugging in Proposition \ref{prop:relationship} to \ref{thm:SiUB}, we see that $|S_i|$ is upper bounded by a constant.
\end{remark}

With these results, we can now characterize the computational complexity of Algorithms \ref{alg:thresholding} and \ref{alg:pruning}, for each node.

For each node, the selected neighborhood $S_i$ always at least as large as $\Delta_i$, but is upper bounded by a constant that's independent of the graph size at the same time. This implies a computational complexity of $\mathcal{O}(n)$ for a set of fixed parameters, since Algorithm \ref{alg:thresholding} iterates at most $\Delta$ rounds, which is upper bounded by a constant, and in each round, at most $n$ conditional covariances are computed. For the pruning algorithm, the computation of each $\hat{\Sigma}_{ij|S_i^{(k)}}$ involves inverting $\hat{\Sigma}_{S_i^{(k)}S_i^{(k)}}$, requiring at most $\mathcal{O}(|S_i|^3)$ computational complexity. Since $|S_i|$ can be upper bounded by a constant independent of $n$, the computational complexity of the pruning algorithm given $\alpha$, $d_{min}$, $d_{max}$, $a$, is essentially $\mathcal{O}(1)$.

\subsection{Correctness}

We now show that, with the threshold chosen as in Algorithm \ref{alg:thresholding}, $S_i$ contains the neighborhood of node $i$, $\mathcal{N}_i$, with high probability This conclusion is based on a lower bound for the absolute value of maximum conditional covariance between a node and its undiscovered neighbors at any point, conditioned on the estimated neighborhood $S_i$ at that point. This property for Gaussian graphical models that meet Assumption \ref{asm:asm} is introduced as follows.

\begin{lemma}
\label{lm:normalizedlb}
Under Assumption \ref{asm:asm}, denote the estimated neighborhood of node $i$ at any point by $S_i$, and assume there are $K$ neighbors of node $i$ undiscovered. Then
\begin{align}
\label{eq:generalizedlb}
\max_{j\in\mathcal{N}_i\backslash S_i}\Sigma_{ij|S_i}^2\geq\frac{1}{K}\frac{\Vert J_{i,\mathcal{N}_i}\Vert_2^2}{d_{ii}^2(d_{ii}(1+\alpha)d_{max}-\Vert J_{i,\mathcal{N}_i}\Vert_2^2)^2},
\end{align}
where $\mathcal{N}_i\backslash S_i=\{j\in V: j\in\mathcal{N}_i,j\not\in S_i\}$ and $d_{ii}$ is the diagonal element of $J$ matrix corresponding to node $i$.
\end{lemma}
Proof is in Appendix \ref{lm:normalizedlbp}.

With the above lemma, we can design the threshold according to the following Theorem.
\begin{theorem}
\label{thm:threshold}
Assume that a Gaussian graphical model satisfies Assumption \ref{asm:asm}. Then, for any node $i$ and any estimated neighborhood $S_i$,
\begin{align}
\max_{j\in\mathcal{N}_i\backslash S_i}|\Sigma_{ij|S_i}|\geq \frac{a}{d_{max}(d_{max}^2(1+\alpha)-a^2)}.
\end{align}
\end{theorem}
\begin{proof}
The proof follows directly from Lemma \ref{lm:normalizedlb}, observing that $\Vert J_{i,\mathcal{N}_i\backslash S_i}\Vert_2^2\geq Ka^2$, and $K\geq 1$.
\end{proof}

If we further restrict the graph to be free from triangles, a tighter bound can be obtained as follows.

\begin{corollary}
\label{cor:trianglefree}
Assume that a Gaussian graphical model satisfies Assumption \ref{asm:asm}, and that the graph does not contain triangles. Then, for any node $i$ and any estimated neighborhood $S_i$,
\begin{align}
\max_{j\in\mathcal{N}_i\backslash S_i}|\Sigma_{ij|S_i}|\geq \frac{a}{d_{max}(d_{max}^2-a^2)}.
\end{align}
\end{corollary}
\begin{proof}
See Appendix \ref{cor:trianglefreep}.
\end{proof}

For the normalized case, the right hand side further simplifies to $a(1-a)^{-0.5}$, which is intuitively correct since a larger value of $a$ makes learning the neighborhood easier as it ``separates" the non-zero entries from the zero entries. Also note that the conditional covariance involved in the proof are exact, which indicates that if we have infinite samples, then Algorithm \ref{alg:thresholding} will return at least 1 neighbor per round. When we have only finitely many samples, the performance of the algorithm depends on how well the empirical conditional covariance concentrates around the conditional covariance, which is the topic of the next subsection.

The intuition for walk summable Gaussian graphical models directly provide the proof for the correctness of Algorithm \ref{alg:pruning}.

\begin{theorem}
Assume a Gaussian graphical model satisfies Assumption \ref{asm:asm}, and for node $i$, $\mathcal{N}_i\subseteq
S_i$. Let $\Gamma=\Sigma_{i,S_i}\Sigma_{S_i,S_i}^{-1}$, and assume that the $j$-th element of $S_i$ corresponds to node $s_i(j)$ in the graph. Then, $\Gamma_j=0$ if $s_i(j)\not\in \mathcal{N}_i$, and $\Gamma_j=-J_{i,s_i(j)}$ if $s_i(j)\in\mathcal{N}_i$.
\end{theorem}
\begin{proof}
The proof follows directly by noticing that when $\mathcal{N}_i\subseteq S_i$, we have $\Sigma_{i,S_i}=-J_{i,S_i}\Sigma_{S_i,S_i}$.
\end{proof}

This result shows that we can set the pruning threshold $\tau^p\in(0,a)$, and it will prune all the non-neighbors and preserve all the neighbors asymptotically. The concentration results involved in working with finite number of samples are differed to the next subsection.

\subsection{Sample Complexity}
\label{sec:correctness}

The structural consistency of our algorithm, along with sample complexity, is stated in the following result.

\begin{theorem}[Structural consistency]\label{thm:sparsistency}

For a Gaussian graphical model satisfying Assumption \ref{asm:asm}, denote the thresholds selected by Algorithms \ref{alg:thresholding} and \ref{alg:pruning} as $\tau$ and $\tau^p$, respectively. Then, given $N$ i.i.d. samples, there exists universal constants $C_3,C_4,C_5$ such that when $N$ scales as
\begin{align}
N>C_3a^{-2}\log n,
\end{align}
with probability at least $1-C_4\exp(-C_5N)$, the combination of Algorithm \ref{alg:thresholding} followed by Algorithm \ref{alg:pruning} returns the actual neighborhood of node $i$.
\end{theorem}
\begin{proof}
See Appendix \ref{thm:sparsistencyp}.
\end{proof}

This indicates that if the number of samples scales as $\Omega (a^{-2}\log n)$, then the structural consistency is guaranteed.

\subsection{Scalability}

We finally show that Assumption \ref{asm:asm} required for Algorithm \ref{alg:thresholding} does not affect the scalability of the graph. This is a question which arises naturally from Assumption \ref{asm:asm}, since we required $\alpha$ walk summability and lower bounded the absolute values of off-diagonal non-zero entries. It also arises from the the result of Theorem \ref{thm:SiUB}, namely, since $|S_i|$ is upper bounded by a constant times $\Delta_i$, where $\Delta_i$ is also upper bounded by a constant, is it possible that under Assumption \ref{asm:asm}, the graph size is restricted to be small?

To answer this question, we present the following result, which implies that there exists graphs of arbitrary size under a fixed set of parameters in Assumption \ref{asm:asm}. The results are given in Propositions \ref{prop:scale}, in which we assume that $\Delta$ is tight, i.e., it is equal to the largest node degree in the graph.
\begin{proposition}[Sufficient condition for scaling the graph]\label{prop:scale}

Given the parameters in Assumption \ref{asm:asm}, there exists $\alpha$ walk summable Gaussian graphical models of arbitrary size if
\begin{align}
1\leq\Delta<\frac{d_{min}\alpha}{b}.
\end{align}
\end{proposition}
\begin{proof}
See Appendix \ref{prop:scalep}.
\end{proof}

\section{Simulation}
\label{sec:numerical}

In this section, we illustrate the performance of our algorithms and the conclusions drawn previously with the help of numerical methods.

\subsection{The forward-backward mutual information test}

We simulated the performance of the mutual information test, forward backward greedy algorithm in \cite{johnson2011high}, and the Lasso method on the following graphs types: chain, star, grid, diamond, and randomly generated . The threshold of the greedy algorithm is set to be relatively large, since small thresholds permit the forward greedy algorithm to select a much larger neighborhood than the actual one, in which case the computational complexity of the pruning would increase. The threshold is set to be fixed, or varying such that $\epsilon\propto\log(n)/N$ decreases as a function of $N$, as  in \cite{johnson2011high}, so that when enough samples are observed, $\sqrt{32\rho\epsilon K_r/C_{min}}$ is lower than the smallest entry of $|\Sigma_{r,-\{r\}}\Sigma^{-1}_{-\{r\}}|$, and the structural consistency is guaranteed by Theorem \ref{ther}. We generated the entries of the inverse covariance matrix randomly. The simulation for  Lasso uses the code in \cite{glmnet}. To encourage sparsity of the results obtained by Lasso, we use the largest regularizer possible such that the mean square error is within 1 standard error of the minimum mean square error, determined by the standard $k$-fold cross validation. For all the experiments, Lasso took up the majority of time, while the greedy algorithms were fast. Thus, we compare the results of Lasso for graphs of relatively small size, and for larger sized graphs, we compare the performance of Algorithm \ref{algorithm} and the greedy forward backward algorithm of \cite{johnson2011high}.
\\
We use two metrics to compare the the algorithms:  (1) success rate, defined as the portion of nodes in the graph whose neighborhood is correctly estimated, averaged over 100 trials, and (2) the accuracy of the test measured by $1-|\hat{A}\ \Delta\ A|/|A|$, where $A$ is the true support of the inverse covariance matrix and $\hat{A}$ is the estimated support. Note that $A\ \Delta\ B:=\{(i,j):A_{i,j}\neq B_{i,j}\}$ when $A$ and $B$ are the adjacency matrices of two graphs.
\\
We first compare the performances of the mutual information test, forward-backward greedy algorithm, and Lasso on the chain graph  ($n=10,d=2$), the star graph ($n=10,d=2$), the grid graph  ($n=9,d=4$), and the diamond graph ($n=4,d=3$). We chose $C_{min}$ to be $0.1$, and $\rho$ ranged from 3 to 10. The threshold was set as $\epsilon=c\log(n)/N$, where $c$ is the tuning parameter. In backward pruning  $\nu=0.5$. For these special graphs, we can see in Figures \ref{fig:specials1} and \ref{fig:specials2} that the mutual information test behaves as good as the forward-backward greedy algorithm. Greedy approaches have similar or better performance than Lasso with much lower computational complexity for small graphs. Much better performance was observed by \cite{johnson2011high} when the threshold decreased as a function of the sample size for larger sized graphs, in both computational complexity and sample complexity.
\begin{figure}[tb]
\minipage{0.56\textwidth}
\hspace{-8.5mm}  \includegraphics[width=\linewidth]{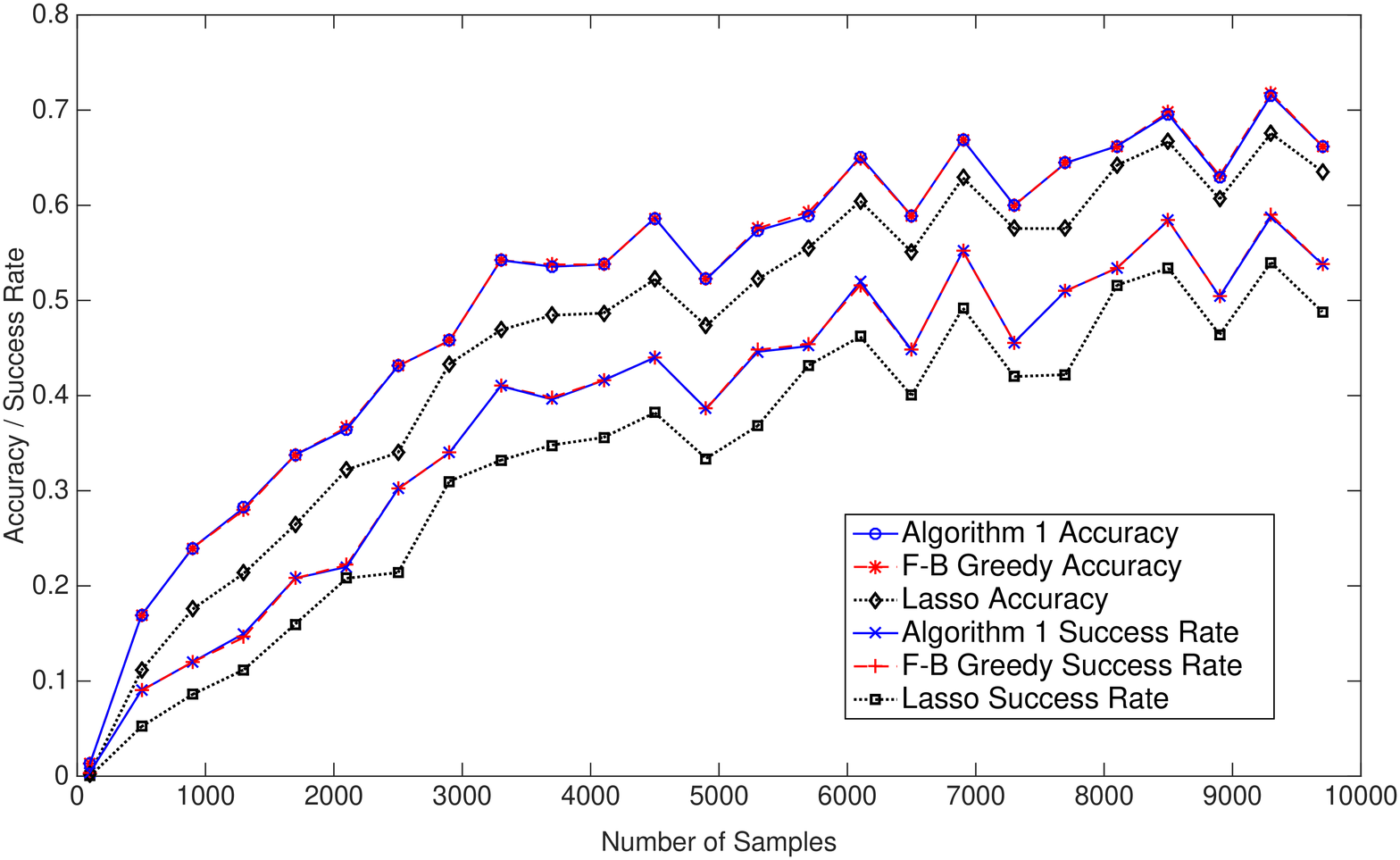}
\hspace{-11mm}  \caption*{(a) Chain of length 10.}
  \label{fig:chain}
\endminipage\hfill
\minipage{0.56\textwidth}
  \hspace{-13mm}  \includegraphics[width=\linewidth]{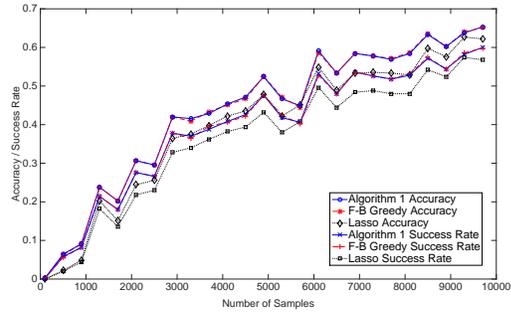}
 \hspace{-6mm} \caption*{(b) Star of size 10.}
  \label{fig:star}
\endminipage\hfill
\caption{Performance comparison between greedy algorithms and Lasso with decreasing threshold}\label{fig:specials1}
\end{figure}
\begin{figure}[tb]
\minipage{0.56\textwidth}
\hspace{-8.5mm}
\includegraphics[width=\linewidth]{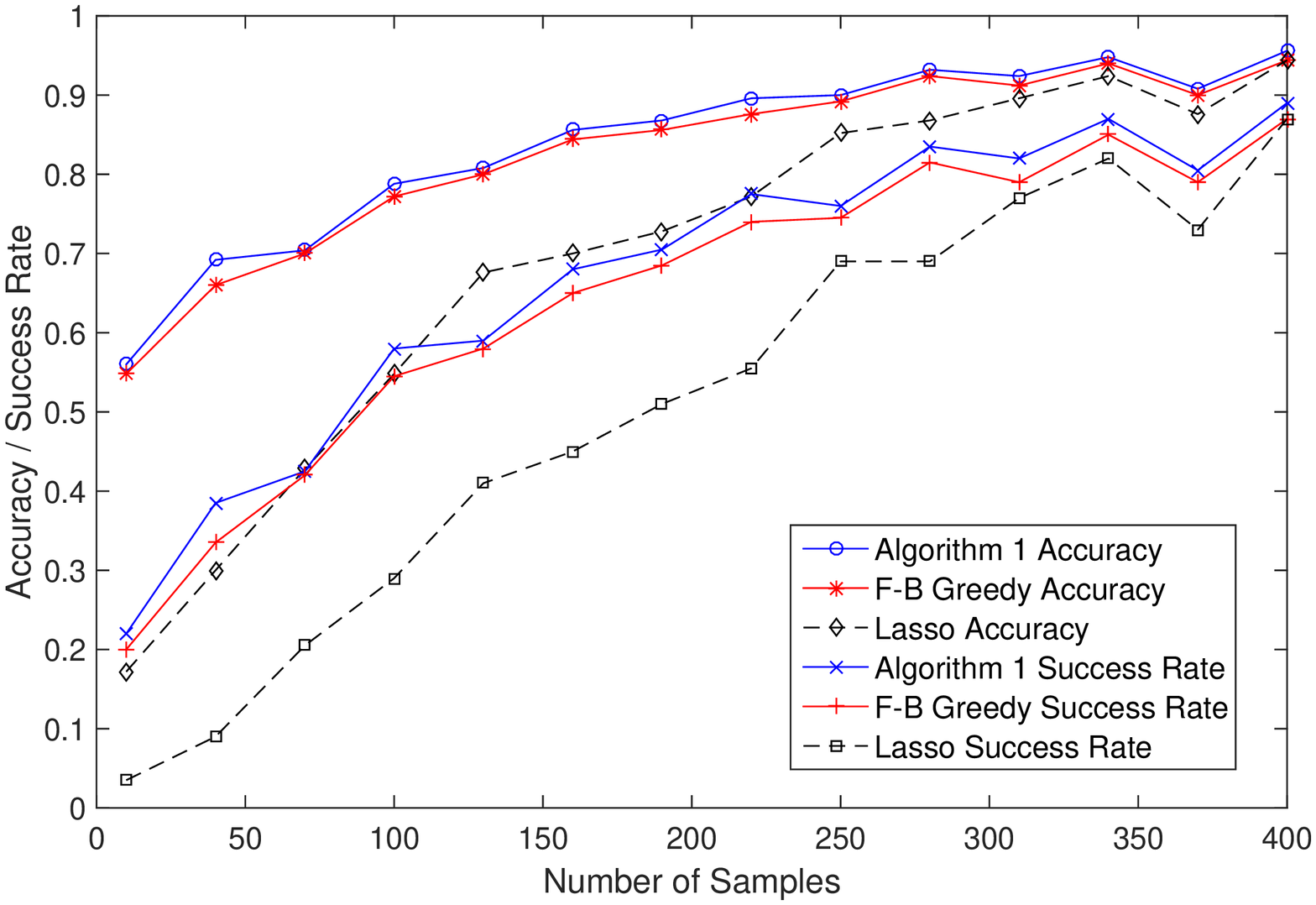}
\hspace{-1mm}  \caption*{(a) Diamond Graph.}
  \label{fig:diamond}
\endminipage\hfill
\minipage{0.56\textwidth}
  \hspace{-13mm}  \includegraphics[width=\linewidth]{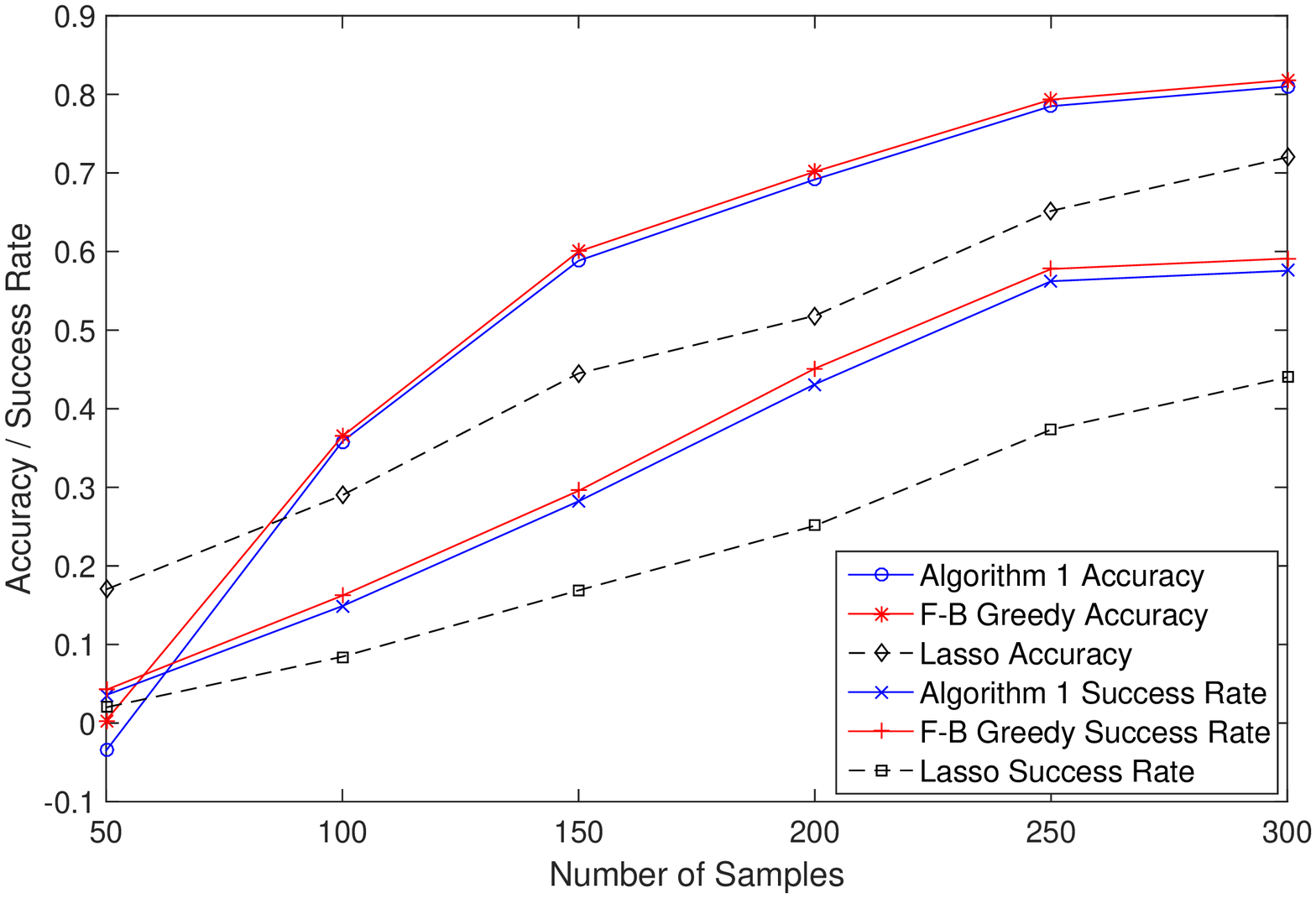}
 \hspace{-6mm} \caption*{(b) $3\times 3$ grid.}
  \label{fig:grid}
\endminipage\hfill
\caption{Performance comparison between the greedy algorithms and Lasso with fixed threshold}\label{fig:specials2}
\end{figure}
\\
We next compared the performance on random graphs of size 10 ($d=6$) and 20 ($d=13$), with average number of edges for each instance around 20 and 51, respectively. $\rho$ is set to at least $10$ to allow easier generation of inverse covariance matrix satisfying the restricted eigenvalue constraints. The results are shown in \ref{fig:random}. It can be seen that Algorithm \ref{algorithm} is slightly superior to the forward-backward greedy algorithm, when the graph becomes denser. This is mainly due to the fact that the forward step of our algorithms uses the conditional mutual information test, which has a higher chance of selecting the correct neighbors.

\begin{figure}[tb]
\minipage{0.56\textwidth}
\hspace{-8.5mm}  \includegraphics[width=\linewidth]{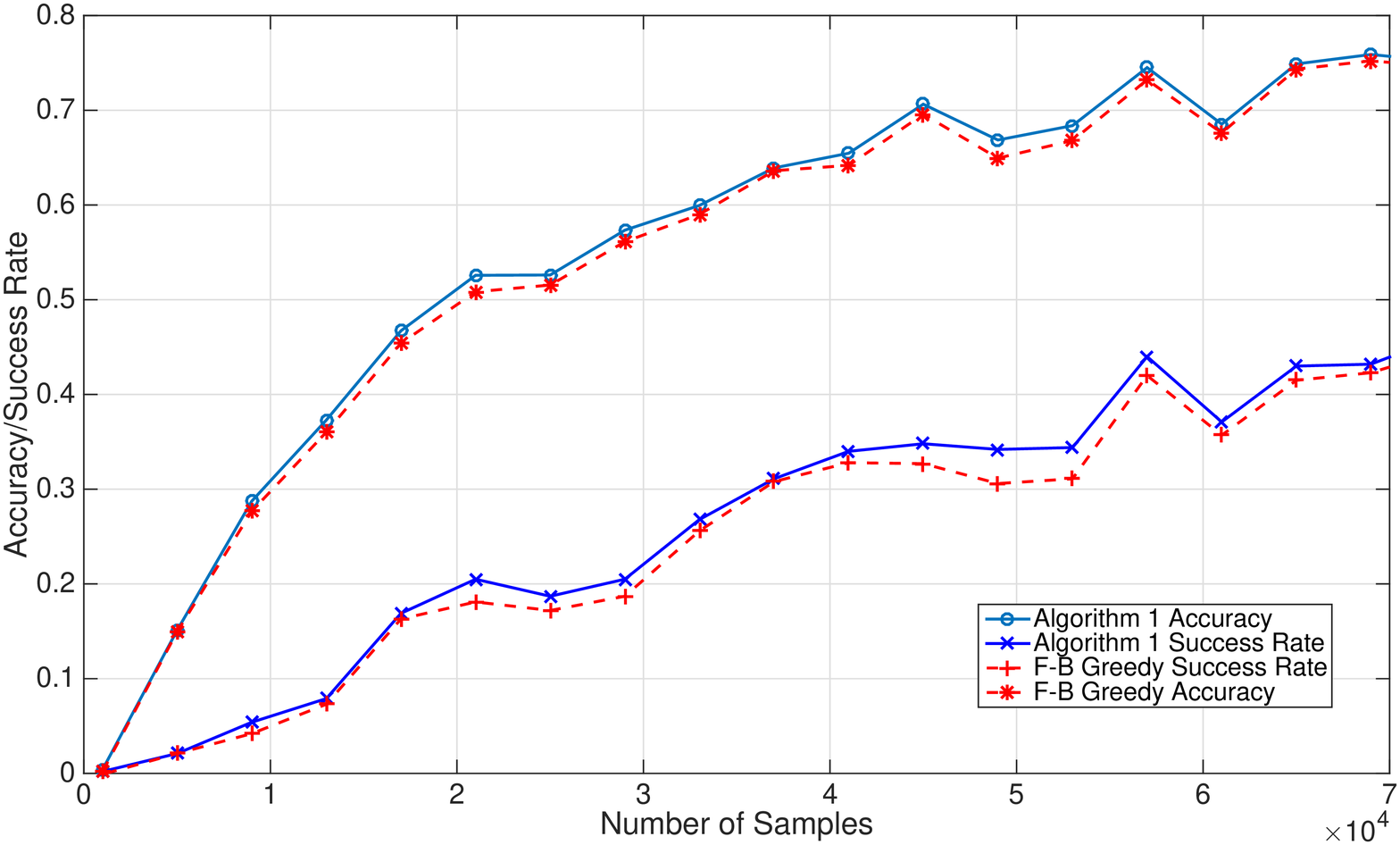}
\hspace{-11mm}  \caption*{(a) $n=10, |E|\approx20$.}
  \label{fig:random10}
\endminipage\hfill
\minipage{0.56\textwidth}
  \hspace{-13mm} \includegraphics[width=\linewidth]{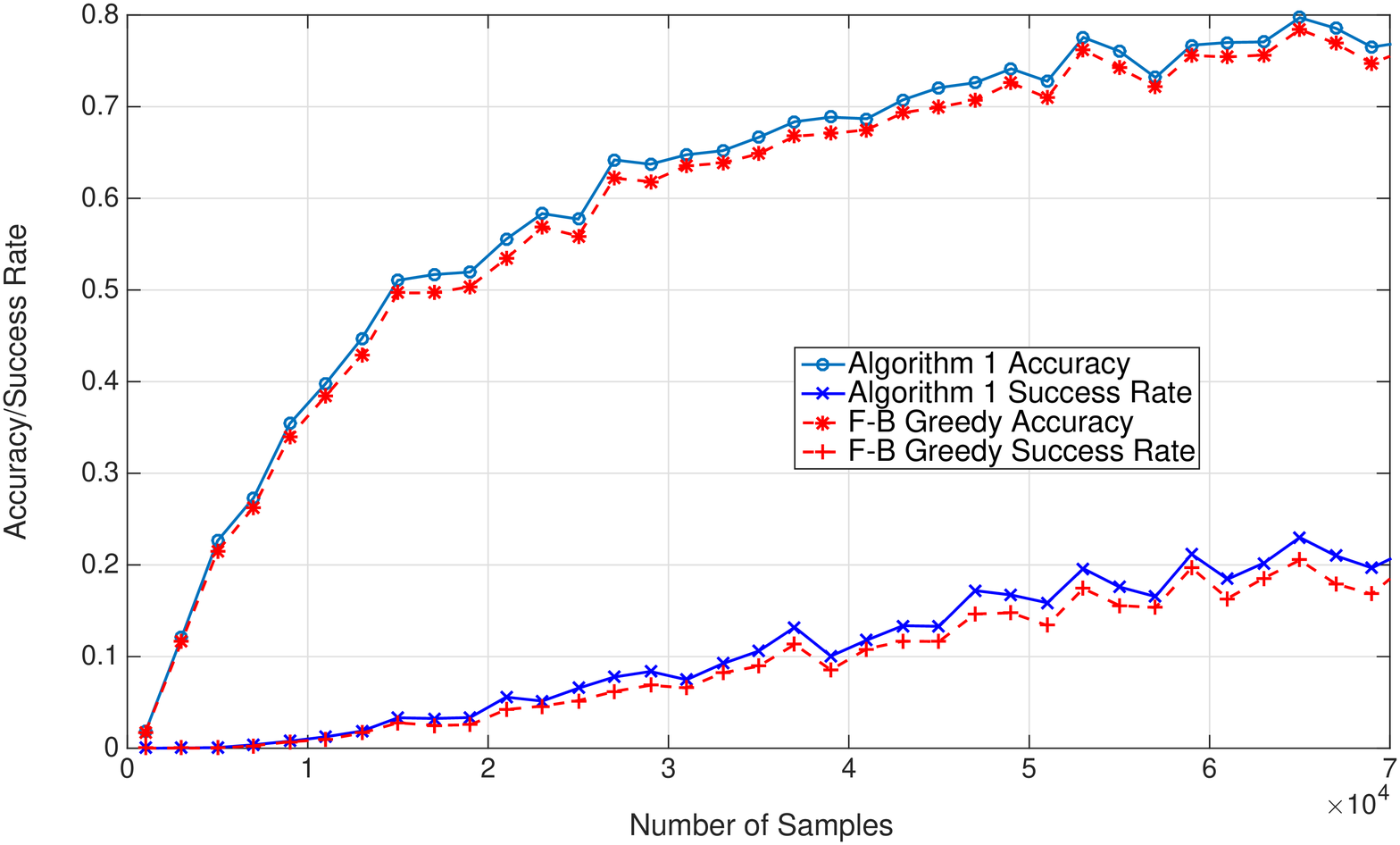}
 \hspace{-6mm} \caption*{(b)  $n=20, |E|\approx51$.}
  \label{fig:random20}
\endminipage\hfill
\caption{Performance comparison between Algorithm \ref{algorithm} and forward-backward greedy algorithm on random graphs with decreasing threshold. }\label{fig:random}
\end{figure}

\subsection{The thresholding algorithm}

\subsubsection{Algorithm efficiency}

We first demonstrate the result of Theorem \ref{thm:SiUB}, assuming that we have infinite samples (and hence the exact covariance matrix $\Sigma$), and that $d_{min}=d_{max}=1$ so that we only have freedom in choosing $\alpha$, $a$ and $b$. When the threshold is selected as in Algorithm \ref{alg:thresholding}, we have
\begin{align}
\frac{|S_i|}{\Delta_i}\leq \frac{(1+\alpha)^2}{(1-\alpha)^2}\cdot\frac{b^2}{a^2},
\end{align}
assuming that $\Delta_i>0$. This implies that the graphical model is easy to learn when (1) $\alpha$ is small, and (2) when $b/a$ is not too large. We hence plot $|S_i|/\Delta_i$ as a function of $\alpha$ for different $b/a$ ratios. As can be seen from Figure \ref{fig:SiUB}, the graph is easier to learn when $b/a$ is small and when $\alpha$ is small. In addition, decreasing value of $a$ does not increase the hardness of learning the graph as long as $b/a$ is fixed. Finally, we point out that even when the upper bound is large, it can be seen from the numerical results shown in later sections that the actual size of $|S_i|$ is small.

\begin{figure}[tb]
\def\tabularxcolumn#1{m{#1}}
\begin{tabular}{cc}
\subfloat[Upper bound of $|S_i|/\Delta_i$ as a function of $\alpha$ for different values of $b/a$.]{\includegraphics[width=0.48\linewidth]{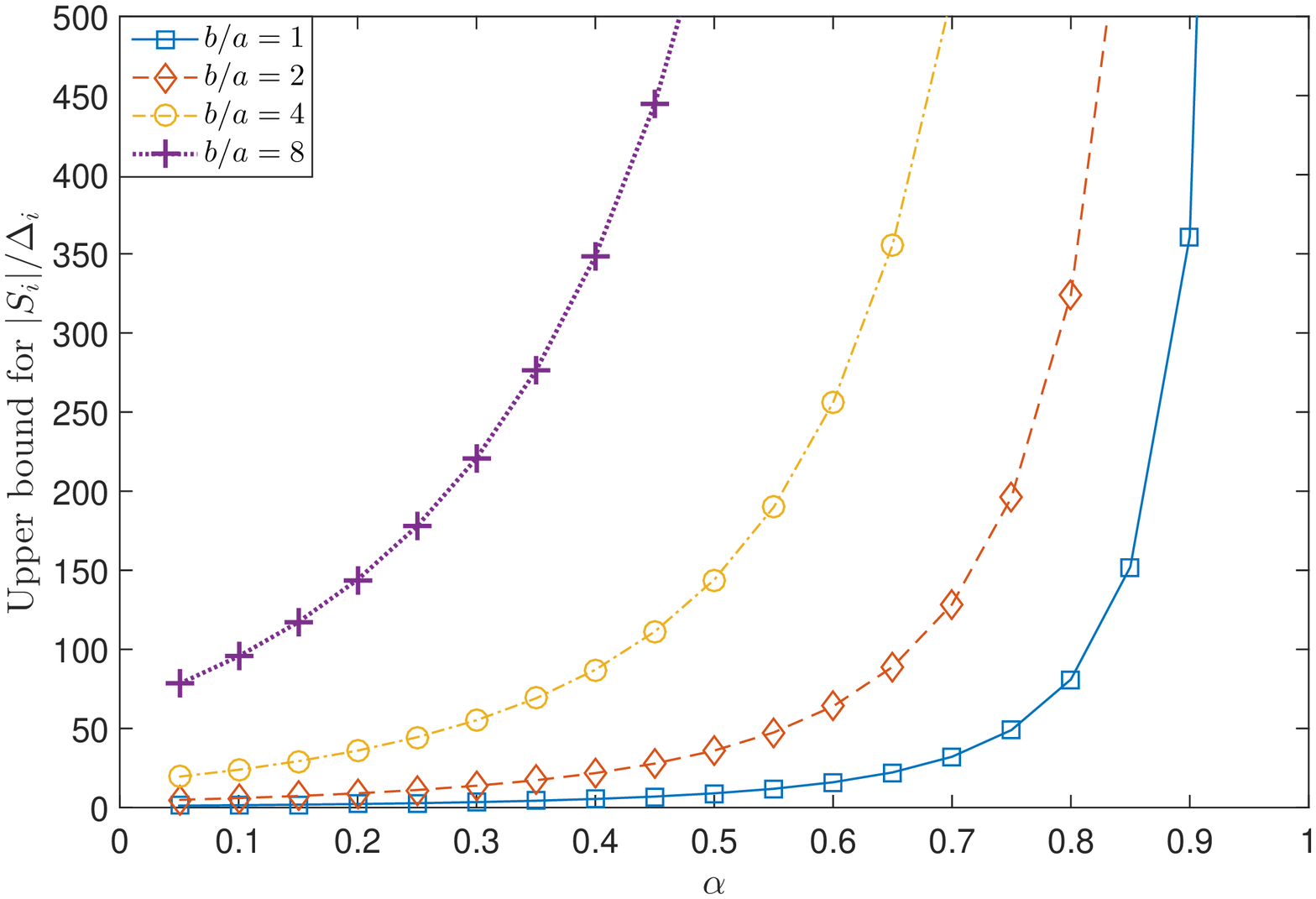}\label{fig:SiUB}} 
   & \subfloat[An example comparing the threshold adopted by the oracle and Algorithm \ref{alg:thresholding}. It can be seen that the looseness of the algorithm allows false neighbors to be selected.]{\includegraphics[width=0.48\linewidth]{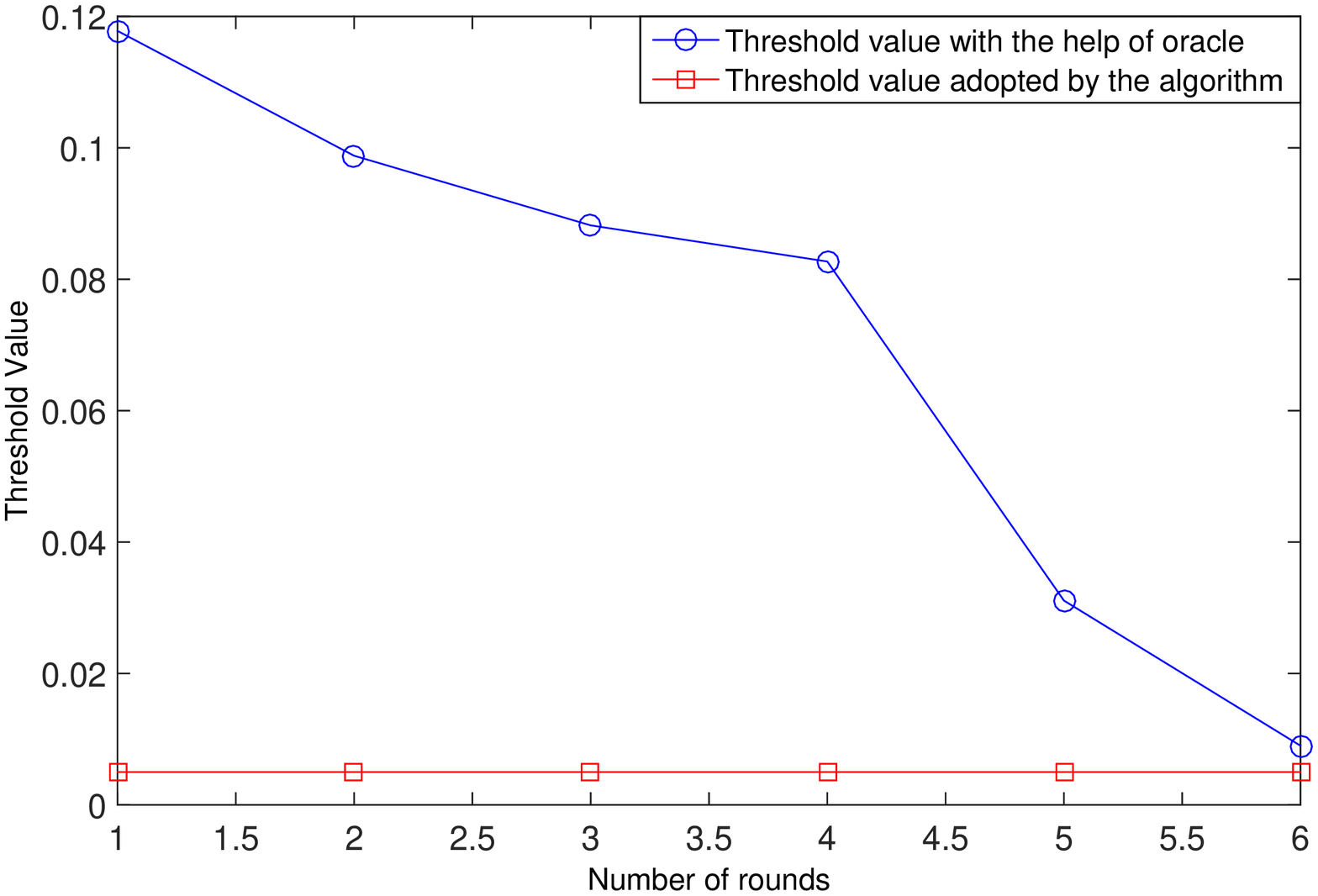}\label{fig:threshold}}
\end{tabular}
\caption{Upper bound for $|S_i|/\Delta_i$ and an illustration of the threshold used by Algorithm \ref{alg:thresholding} and the lower bound provided by the corresponding lemma.}
\end{figure}

\subsubsection{Normalized case: a random instance}

For simplicity and space limit, we only demonstrate the normalized case without triangles, the case where the diagonal entries of $J$ are all ones. This frees us from setting different $d_{min}$ and $d_{max}$, and the results are easy to track, although we also note that the result for generalized case will be slightly worse than the normalized case as well. We randomly generate graphical structures, rejecting those instances that contain triangles, and then generate the off-diagonal entries of the upper triangle matrix using i.i.d. Gaussian distribution with $a=0.01$. We scale the off-diagonal entries to make sure that the spectral radius of $|R|$ is below a certain level of $\alpha$, and reject those that violate the entry wise lower bound.

Since the actual lower bound adopted is quite loose, which we shall see later, it is likely that Algorithm \ref{alg:thresholding} selects a superset of the actual neighborhood with very high probability. Hence, the following algorithm can be implemented right after Algorithm \ref{alg:thresholding}. Algorithm \ref{alg:pruningbysymmetry} simply checks whether the adjacency matrix obtained is symmetric. It can be easily seen that if Algorithm 1 succeeds, then $i$ and $j$ must be simultaneously in each other's estimated neighborhood if they are actual neighbors.

\begin{algorithm}[H]
\caption{Pruning by Symmetry}
\label{alg:pruningbysymmetry}
    \begin{algorithmic}[1]
        \STATE {\it Input:} $S_1,S_2,...,S_n$.
        \STATE {\it Output:} $S'_1,S'_2,...,S'_n$.
        \STATE { Initialization: $A\leftarrow zeros(n,n)$.}     
        \STATE {Set $A(i,S_i)\leftarrow 1$ for all $i=1,...,n$}
        \STATE {For all $(i,j)$, if $A(i,j)\neq A(j,i)$, set $A(i,j)\leftarrow 0$, and $A(j,i)\leftarrow 0$.}
        \STATE {Set $S'_i\leftarrow \text{find}(A(i,:)\neq 0)$ for all $i=1,...,n$}
    \end{algorithmic}
\end{algorithm}

\begin{figure}[tb]
\def\tabularxcolumn#1{m{#1}}
\begin{tabular}{cc}
\subfloat[True graph]{\includegraphics[width=0.48\linewidth]{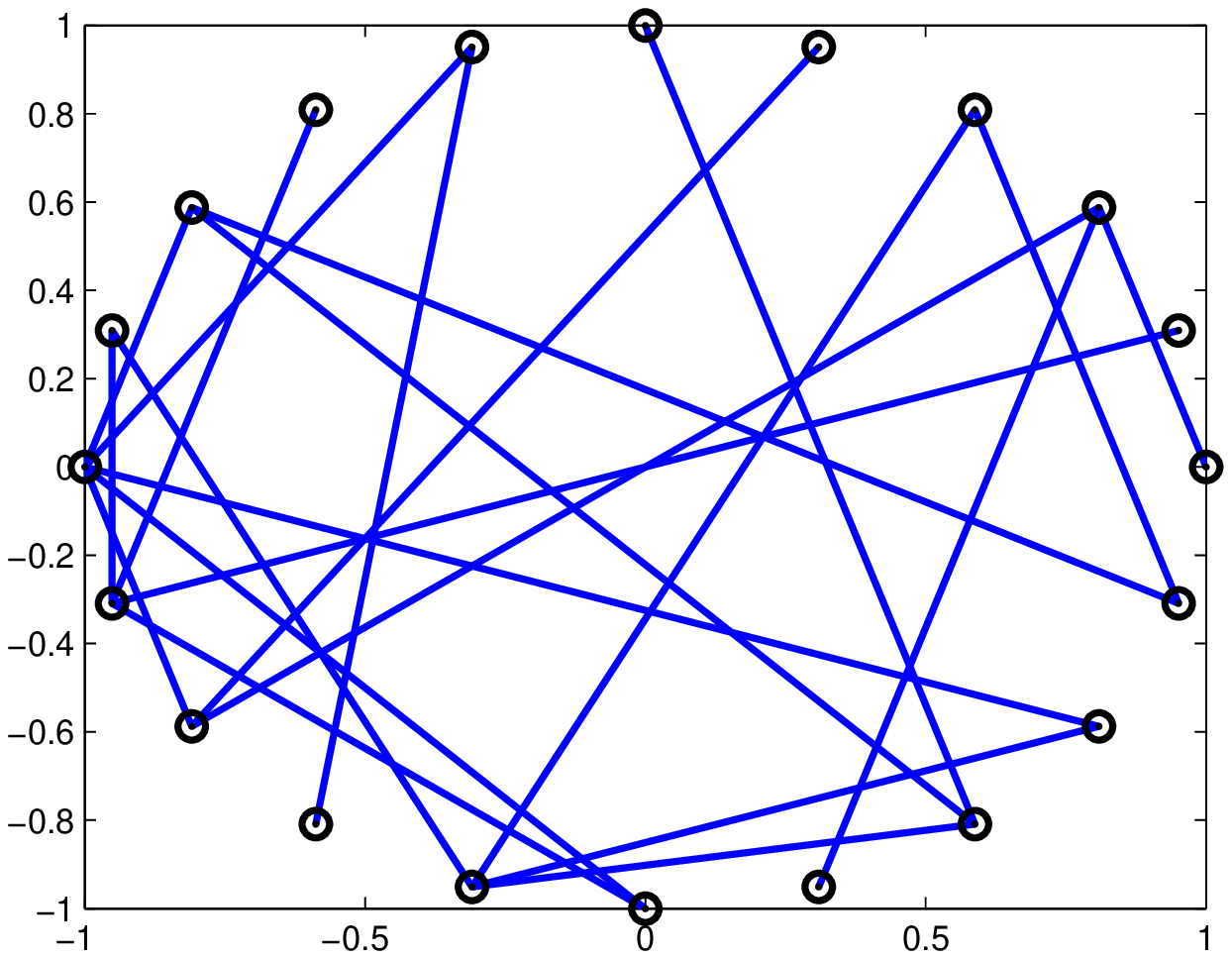}} 
   & \subfloat[Oracle Version of Algorithm \ref{alg:thresholding}, accompanied by Algorithm \ref{alg:pruningbysymmetry}.]{\includegraphics[width=0.48\linewidth]{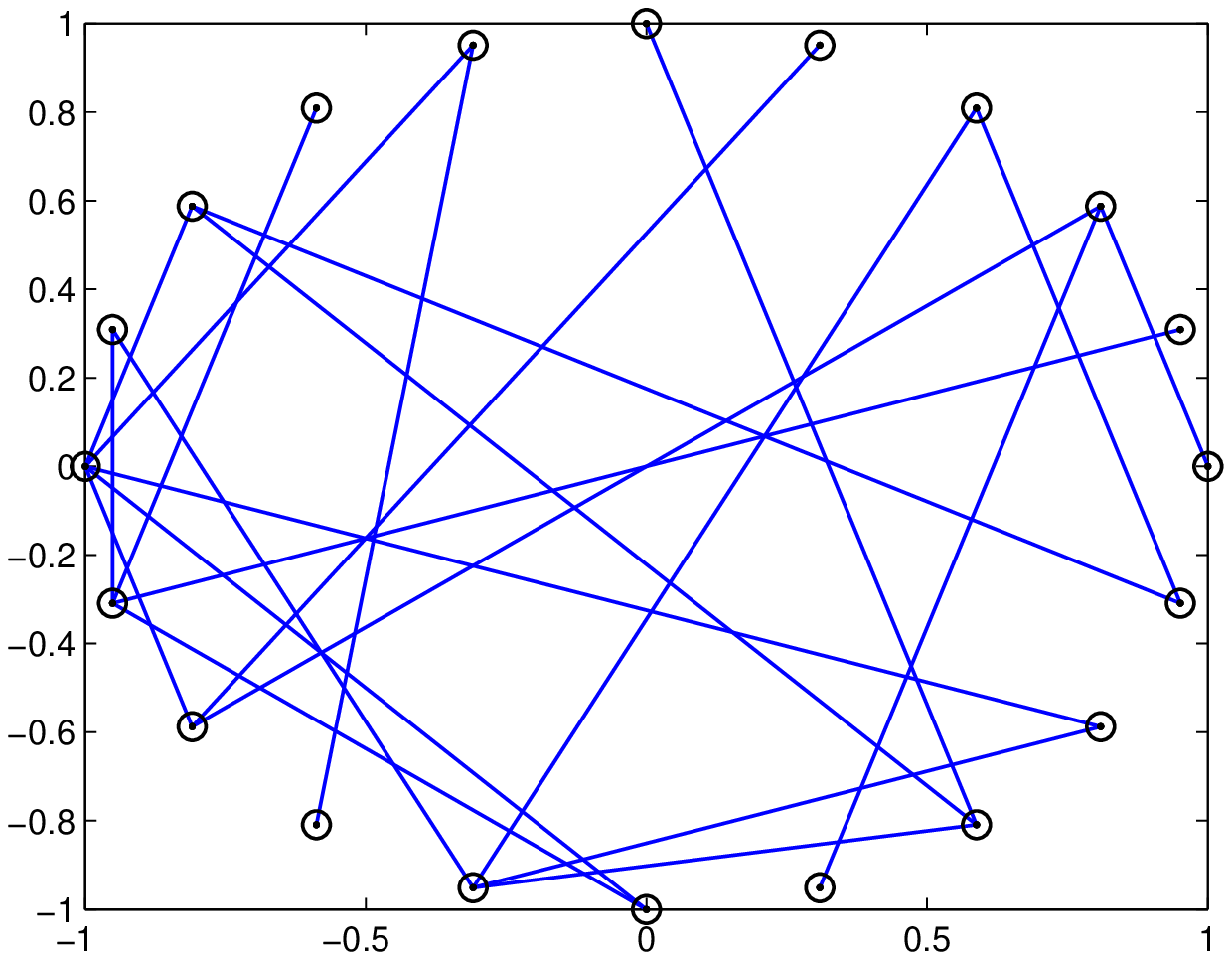}}\\
\subfloat[Algorithms \ref{alg:thresholding} and \ref{alg:pruningbysymmetry} with 1E6 samples]{\includegraphics[width=0.48\linewidth]{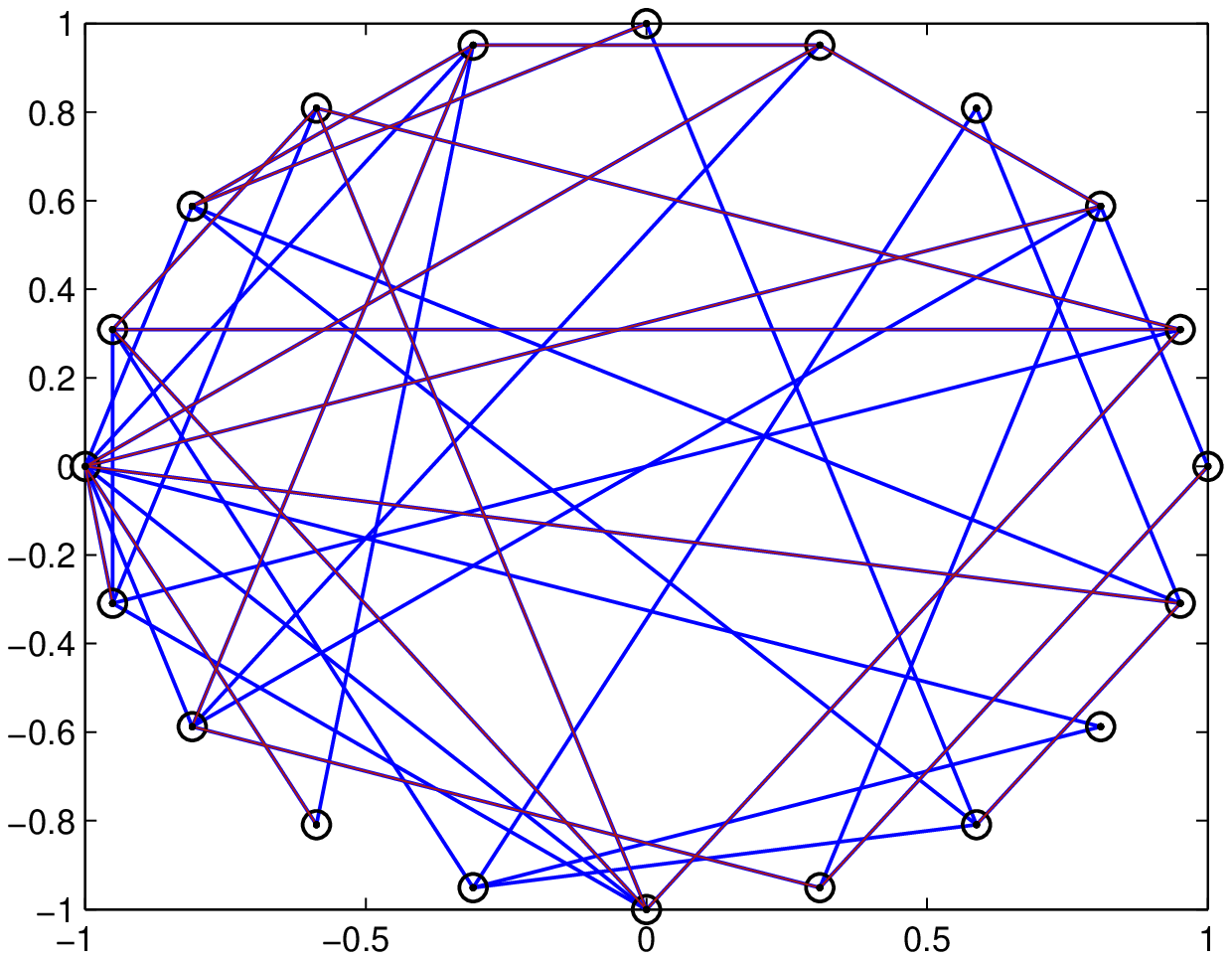}} 
   & \subfloat[Algorithms \ref{alg:thresholding}, \ref{alg:pruning}, and \ref{alg:pruningbysymmetry}, with 1E6 samples]{\includegraphics[width=0.48\linewidth]{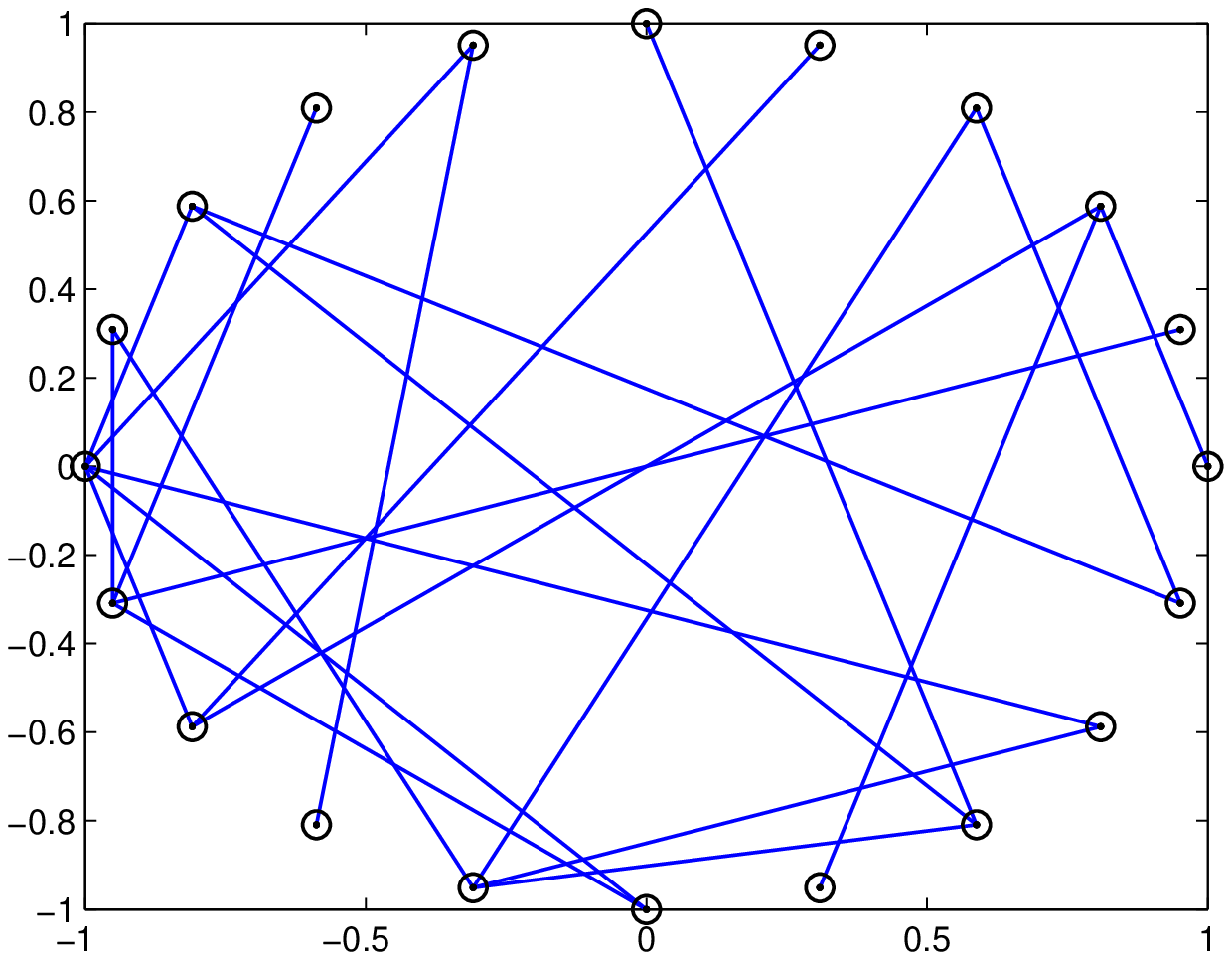}}\\
\end{tabular}
\caption{One random instance with 20 nodes. The true graph contains 22 edges. The oracle version and the Algorithm \ref{alg:thresholding} accompanied by Algorithm \ref{alg:pruning} for pruning returns the correct graph. The Algorithm \ref{alg:thresholding} alone returns a graph containing all 22 true edges and 19 false edges. The pruning algorithm \ref{alg:pruningbysymmetry} is applied automatically. We set $\alpha=0.4$, $a=0.01$, $b=0.28$ and $\Delta=10$. For triangle free graphs, we substitute the threshold with \ref{cor:trianglefree}.}\label{fig:randominstance20}
\end{figure}

We first demonstrate the effectiveness of the threshold algorithm accompanied by rough pruning with Algorithm \ref{alg:pruningbysymmetry}, without applying Algorithm \ref{alg:pruning} for finer pruning. A random sample graph is generated, and the results of the algorithms running on graphs of size 20 are shown in Figure \ref{fig:randominstance20}. We see that the combination of Algorithms \ref{alg:thresholding} and \ref{alg:pruningbysymmetry} manages to find all the true neighbors, while not selecting too many false ones. Further pruning by Algorithm \ref{alg:pruning} gives the correct graph. It can be seen that the actual size of $|S_i|$ is much lower than the upper bound. Meanwhile, to compare with the actual algorithm, we also provide an oracle version of Algorithm \ref{alg:thresholding} by substituting the threshold, which is designed according to Corollary \ref{cor:trianglefree} for triangle free graphs, by the lower bound provided in Appendix \ref{cor:trianglefreep} that corresponds to Lemma \ref{lm:normalizedlb}, and providing all the necessary information including $K$ and $\Vert J_{i,\mathcal{N}_i\backslash S_i}\Vert_2^2$, and $\Sigma$ to the algorithm. It can be seen that the lower bound provided in \ref{lm:normalizedlb} (here it's the corresponding version for triangle free graphs), is quite tight. The difference between the actual bound adopted by Algorithm \ref{alg:thresholding} and its oracle version for the triangle free graphs, is depicted in Figure \ref{fig:threshold}.

\subsubsection{Normalized case: probability of success}

We next plot the probability of success for our algorithm, combining Algorithms 1 to 3, as a function of the number of samples required. The simulation is carried out on 20 node random graphs averaging over 100 instances.

We compare the result of our algorithm to the forward-backward greedy algorithm of \cite{johnson2011high}, which has been shown to outperform the method of Neighborhood Lasso. The main idea of the forward-backward greedy algorithm is to first repeatedly select the node into $S_i$ that minimizes a loss function, which is the distance of $X_i$ to the linear vector space spanned by $X_j$ with $j\in S_i$ in $\mathcal{L}_2(\Omega,\mathcal{F},P)$, until the change in the loss function is below a certain threshold $\epsilon_s$. After the forward part of the algorithm terminates, the backward part of the algorithm prunes the node that causes least amount of change in the loss function until the change is greater than $\nu\epsilon_s$, where $\nu\in(0,1)$ is chosen arbitrarily.

For the forward-backward greedy algorithm, we set (following their notations) $C_{min}=1/(1+\alpha)$, $d$ be the actual degree upper bound of each generated instance, $\rho=(1+\alpha)/(1-\alpha)$. The forward stopping threshold is set to $\epsilon_s=8c\rho\eta d\log(m)/(k C_{min})$, where $k$ is the number of samples, and $\eta=\lceil2+4\rho^2(\sqrt{\rho^2-\rho}/d+\sqrt{2})^2\rceil$. The tuning parameter $c$, which is used to determine the value of $\epsilon_s$, is set from $10^{-1}$ to $10^{-4}$. Notice that $c$ is a tuning parameter. Even though for different $c$ the algorithm will always be structurally consistent, it has to be manually tuned for different set of parameters ($m$, $\alpha$, $a$ for entry wise lower bound, and $d$ for degree upper bound) in order for the algorithm to converge fast. This can be observed from Figure \ref{fig:PoS20}, where different values of $c$ yields different rates of convergence. In fact, many algorithms require the knowledge of tuning parameters, a study on such problem can be found in \cite{liu2012tiger}. Finally, we point out that the forward-backward greedy algorithm is more powerful when used to select the graphical structure as a whole, compared to neighborhood selection, and that the forward-backward greedy algorithm is likely to be more efficient when the number of samples is small. More details can be found in \cite{johnson2011high}.

We analyze the results are shown in Figure \ref{fig:PoS20}. We also calculated how many edges in total both algorithms selects, shown in Figure \ref{fig:FwdCnt20}, where each algorithm is accompanied by the raw pruning done by Algorithm \ref{alg:pruningbysymmetry}. From the figure, we can see that (1) the performance of the forward-backward greedy algorithm is very sensitive to the tuning parameter $c$, while our algorithm does not involve any notion of tuning parameter; (2) the performance of our algorithm matches the best performance for the forward-backward greedy algorithm with the five choices of $c$; (3) Both our algorithm, and the forward-backward greedy algorithm with the best choice of $c$, are very efficient in the selecting the pseudo neighborhood (although forward-backward greedy algorithm prunes the pseudo neighborhood every time it picks a new node) when the number of samples is large. The reason that our algorithm will select a large pseudo neighborhood when the number of samples is small is due to the choice of the pruning threshold in Algorithm \ref{alg:pruning}, which we arbitrarily set to $10^{-3}$.

\begin{figure}[tb]
\def\tabularxcolumn#1{m{#1}}
\begin{tabular}{cc}
\subfloat[Probability of Success. Both the forward selection parts and backward pruning parts of the proposed algorithm and the forward-backward greedy algorithms are applied (as opposed to the other part of this figure).]{\label{fig:PoS20}\includegraphics[width=0.48\linewidth]{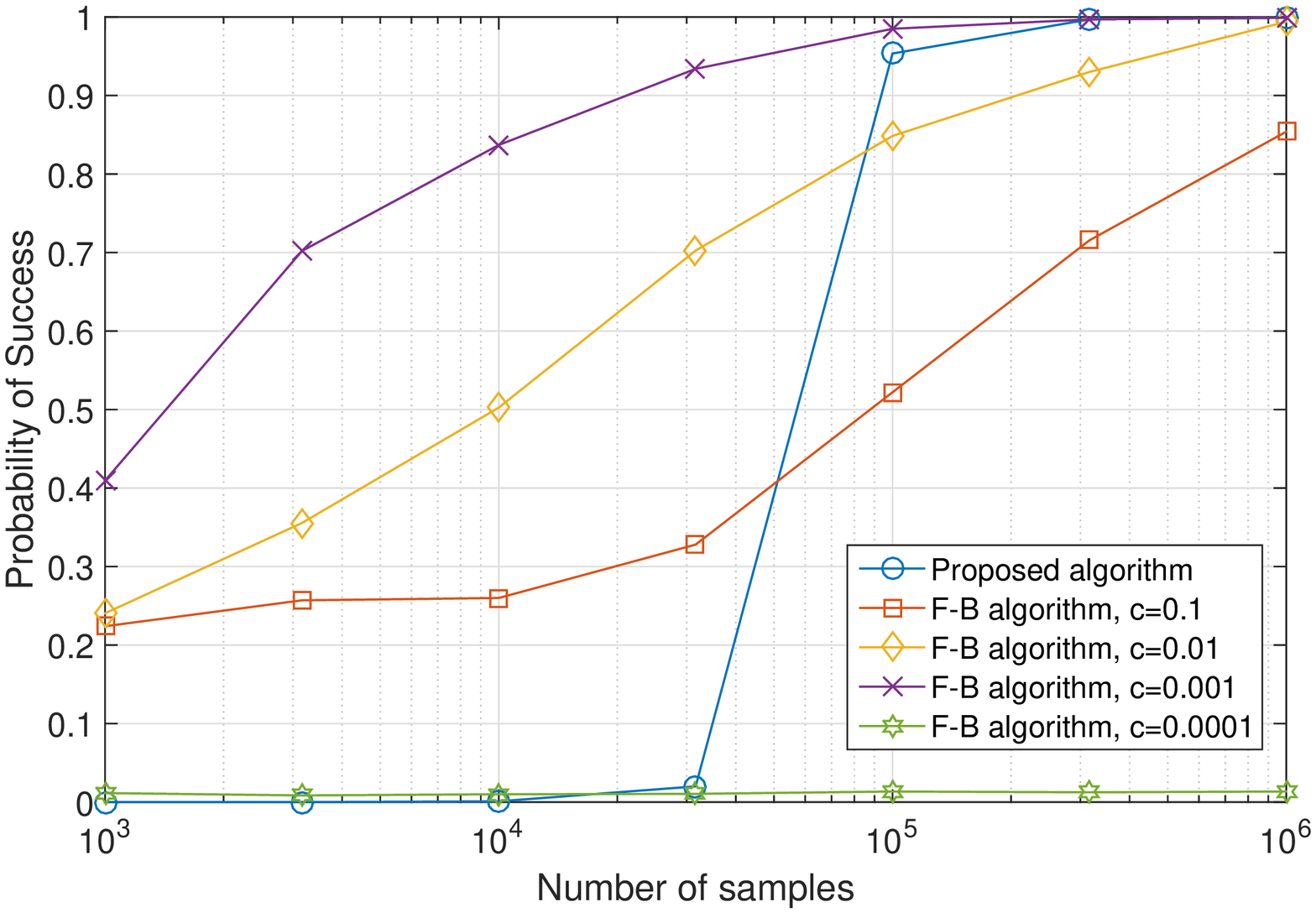}} 
   & \subfloat[Pseudo neighborhood size selected by the forward part of the algorithm, pruned only by symmetry.]{\label{fig:FwdCnt20}\includegraphics[width=0.48\linewidth]{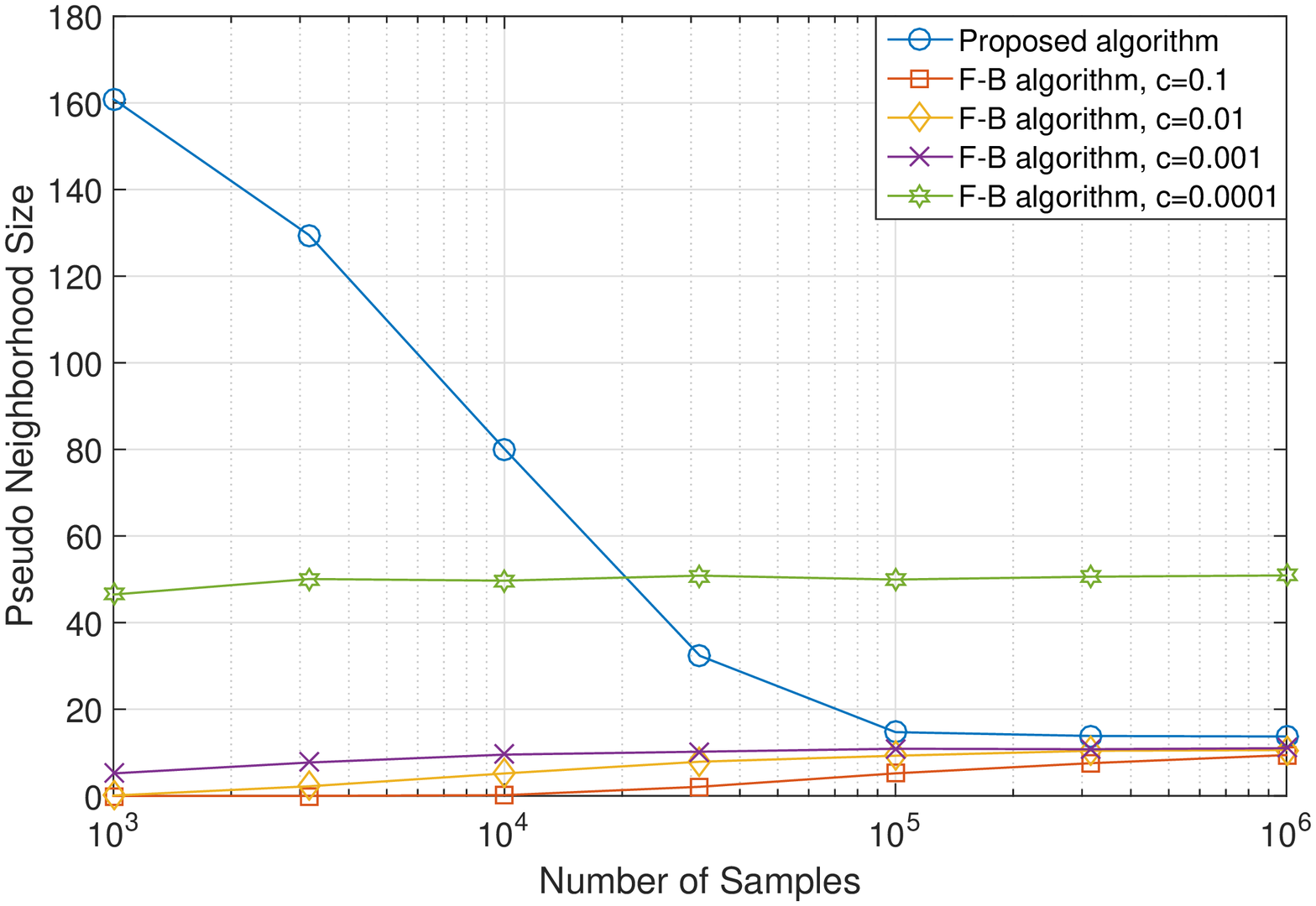}}
\end{tabular}
\caption{Performance evaluation on 20 node random graphs. When $c$ decreases from $10^{(-1)}$ to $10^{(-3)}$, the forward stopping threshold for the forward-backward greedy algorithm is large enough so that the algorithm will not select a very large pseudo neighborhood. When $c=10^{(-4)}$, the threshold becomes too small, and the pseudo neighborhood size becomes very large. }
\end{figure}

\section{Conclusion and future work}

In this paper, we studied the problem of neighborhood selection for walk summable Gaussian graphical models. We presented a novel property for those type of models, which lower bounds the maximal absolute value of the conditional covariance between a node and its undiscovered neighbors. Based on this property, we presented two algorithms which greedily selects the neighborhood by thresholding the conditional covariance, and prunes the potentially false neighbors, respectively. When the graph does not contain any triangles, the bounds can be tightened. We characterized the efficiency of the algorithm in terms of the upper bound for the ratio between the selected pseudo neighborhood by thresholding and the actual degree of the algorithm. We also showed that the number of iterations executed by the algorithm is also upper bounded. We gave computational complexity results for the algorithms we propose, and we presented results on sample complexity and structural consistency of the algorithms when working with finite number of samples. We simulated our algorithm for the triangle free version, and the numerical results showed that the algorithms work well in reality.

Future work includes the following aspects:
\begin{itemize}
\item Can $|S_i|$ be more efficiently upper bounded?
\item Can we develop a similar technique for more generalized type of graphs, mainly Gaussian graphical models that are not walk summable? 
\end{itemize}
The answer for the second question is more fundamental, and we believe it requires a clever way of connecting $J_{i,\mathcal{N}_i}$ and $\Sigma_{i,\mathcal{N}_i}$, and a clever way of exploiting the zero patterns of the $J$ matrix.

\bibliographystyle{plain}
\bibliography{ref}

\section{Appendix}
\label{sec:appendix}

\subsection{Walk Summable GMRF}
\label{sec:WSGMRF}

We start off by introducing several concepts that lead to the concept of walk-summability.

For any information matrix $J$ that is symmetric and positive definite, we denote $J_{norm}$ to be its {\it normalized version}, which can be obtained by letting
\begin{align}
J_{norm,ij}=\frac{J_{ij}}{\sqrt{J_{ii}J_{jj}}}.
\end{align}
An information matrix $J$ itself is said to be normalized if $J=J_{norm}$, i.e., the diagonal elements of $J$ are all ones. Otherwise, we can represent $J$ by
\begin{align}
J=\sqrt{D}J_{norm}\sqrt{D},
\end{align}
where $D$ is the diagonal matrix of $J$, with $D_{ii}=J_{ii}$ and $D_{ij}=0$ for $i\neq j$.

Next, we denote $R$ as the partial correlation matrix, which satisfies $R=I-J_{norm}$. The entry $R_{ij}$ measures $\Sigma_{ij|V\backslash\{i,j\}}$. We denote $|R|$ as the matrix satisfying $|R|_{ij}=|R_{ij}|$.

With these two concepts, we proceed to the definition of walk summability, which is parameterized by $\alpha$.

\begin{definition}[$\alpha$ walk summability \cite{anandkumar2012high}] A GMRF is said to be $\alpha$ walk summable if $\Vert |R|\Vert_2\leq\alpha<1$.
\end{definition}

The most elegant property of the walk summability is that it allows us to relate the covariance between $X_i$ and $X_j$ as the weight sum of all the paths between nodes $i$ and $j$ in the underlying graph specified by $J$, where the edge $(i,j)$ weighs $J_{ij}$ and the weight of a path is the product of all the edge weights along the path. We illustrate this in the following example, which comes in useful later.
\begin{example}[Normalized $J$ matrix]
\label{eg:normalized}
Consider the special case where $J$ is itself normalized, i.e., $J=J_{norm}$. In this case,
\begin{align}
\label{eq:neumman}
\Sigma=J^{-1}=(I-R)^{-1}=\sum_{k=0}^{\infty}R^{k}.
\end{align}
This shows that $\Sigma_{ij}$ is the sum of $R^{k}_{ij}$ for all non-negative integers $k$. Since $R$ preserves the graphical structure of $J$ between different nodes, it's immediately seen that $R^{k}_{ij}$ represents the summation of weight of all the paths from $i$ to $j$ with $k$ hops.
\end{example}

Another interesting property of the walk summable Gaussian graphical model is that it has restricted eigenvalues, which often appears as the condition for exact recovery of the neighborhood (with high probability) using Lasso and greedy methods \cite{johnson2011high}. Here we present it as a lemma.
\begin{lemma}
\label{lm:re}
Assume a Gaussian graphical model with $n$ vertexes is $\alpha$ walk summable, and $J_{ii}\in [d_{min},d_{max}], \forall i\in\{1,2,...,n\}$, then $\forall x\in \mathbbm{R}^n$,
\begin{align}
\label{eq:reC}
(1+\alpha)^{-1}d_{max}^{-1}\Vert x\Vert_2\leq&\Vert \Sigma x\Vert_2\leq (1-\alpha)^{-1}d_{min}^{-1}\Vert x\Vert_2,\\\label{eq:reJ}
(1-\alpha)d_{min}\Vert x\Vert_2\leq&\Vert Jx\Vert_2\leq (1+\alpha)d_{max}\Vert x\Vert_2.
\end{align}
\end{lemma}
\begin{proof}
Denote the $D$ as the diagonal matrix of $J$. Then $J=\sqrt{D}(I-R)\sqrt{D}$. Since $\Vert |R| \Vert_2\leq \alpha$, we know that $-\alpha\leq \lambda_R\leq \alpha$. 
Hence, $(1-\alpha)d_{min}\leq\lambda_J\leq (1+\alpha)d_{max}$. Since $J=\Sigma^{-1}$,
\begin{align}
(1+\alpha)^{-1}d_{max}^{-1}\leq\lambda_{\Sigma}\leq(1-\alpha)^{-1}d_{min}^{-1}.
\end{align}
\end{proof}
This indicates that $\alpha$ walk summability is a stronger condition than the restricted eigenvalue condition, in the sense that it applies to all vectors $x$, while the restricted eigenvalue condition only requires sparse vectors, see for example \cite{johnson2011high}. Hence, the well explored methods such as Lasso and forward-backward greedy algorithms can be readily applied to walk summable Gaussian graphical models. The reverse direction of this relationship, however, may not be true, since we cannot imply $\Vert |R|\Vert_2< 1$ from $-\alpha\leq\lambda_{R}\leq\alpha$, while in order for the model to be walk summable, we must have $\Vert |R|\Vert_2< 1$ \cite{malioutov2006walk}.

In addition, notice that the value of $(1+\alpha)d_{max}/[(1-\alpha)d_{min}]$ characterizes the strength of the constraint on the eigenvalues of $J$ and $\Sigma$. With the same value of $(1-\alpha)d_{min}$, if we decrease one of $\alpha$ and $d_{max}/d_{min}$ while holding the other, then the eigenvalues of $\Sigma$ are restricted to a smaller region.

Finally, for convenience of further reasoning, we point out the following fact.
\begin{corollary}
\label{coro:re}
For an $\alpha$ walk summable Gaussian graphical model, let $S\subset V$. Denote $M^{(1)}=J_{SS|\bar{S}}$, $M^{(2)}=J_{SS}$, $N^{(1)}=\Sigma_{SS|\bar{S}}$, $N^{(2)}=\Sigma_{SS}$. Then $\forall x\in \mathbbm{R}^{n}$, and for $i=1,2$, we have
\begin{align}
(1-\alpha)d_{min}\Vert x\Vert_2&\leq\Vert M^{(i)}x\Vert_2\leq(1+\alpha)d_{max}\Vert x\Vert_2,\\(1+\alpha)^{-1}d_{max}^{-1}\Vert x\Vert_2&\leq\Vert N^{(i)}x\Vert_2\leq(1-\alpha)^{-1}d_{min}^{-1}\Vert x\Vert_2.
\end{align}
\end{corollary}
\begin{proof}
When $i=2$, the proof follows directly from Lemma \ref{lm:re}, by applying the interlacing property of the eigenvalues for a submatrix. When $i=1$, the proof follows similarly by first using the relationships (\ref{eq:conditioning}) and (\ref{eq:marginalization}), and then apply the interlacing property of the eigenvalue and Lemma \ref{lm:re}.
\end{proof}

\subsection{Proof of Lemma \ref{lm:relationship}.}\label{lm:relationship_proof}
\vspace{-1mm}
 Denote the projection of $X_i$ and $X_j$ onto $X_S$ by $X'_i$ and $X'_j$, respectively. Then, $X'_i=(\beta')^TX_S$, where $\beta'=\arg\min_{\beta}\mathbb{E}[( X_i-\beta^TX_S)^2]$ and it can be shown that 
$
X'_i=\Sigma_{i,S}\Sigma_{S,S}^{-1}X_S.
$
Similarly, $X'_j=\Sigma_{j,S}\Sigma_{S,S}^{-1}X_S$. Hence,
\begin{align}\nonumber
&\min_{\alpha}\mathbb{E}[( Y_i-\alpha Y_j)^2]=\mathbb{E}[Y_i^2]-\frac{\mathbb{E}^2[Y_i Y_j]}{\mathbb{E}[Y_j^2]}\\ \nonumber &=\Sigma_{ii}-\Sigma_{S,i}^T\Sigma_{S,S}^{-1}\Sigma_{S,i}-\frac{(\Sigma_{ij}-\Sigma_{S,i}^T\Sigma_{S,S}^{-1}\Sigma_{S,j})^2}{\Sigma_{jj}-\Sigma_{S,j}^T\Sigma_{S,S}^{-1}\Sigma_{S,j}}\notag=\Sigma_{ii|S}-\frac{\Sigma^2_{ij|S}}{\Sigma_{jj|S}},\nonumber
\end{align}
where the last equality uses the expression for conditional covariance between $X_i$ and $X_j$ conditioned on $X_S$. The proof is then completed by noting that
\begin{align}\nonumber
2I(X_i;X_j|X_S)&=\log\frac{\Sigma_{ii|S}\Sigma_{jj|S}}{\Sigma_{ii|S}\Sigma_{jj|S}-\Sigma_{ij|S}^2}=\log\frac{\Sigma_{ii|S}}{\min_{\alpha}\mathbb{E}[( Y_i-\alpha Y_j)^2]}.
\end{align}
\vspace{-4mm}
\subsection{Proof of Theorem \ref{prune}.}\label{proof_prune}
\vspace{-1mm}
As we discussed, the support of $J=\Sigma^{-1}$ identifies the corresponding MRF of a jointly Gaussian system. Hence, a node $j$ dose not belong to $\mathcal{N}_i$ if and only if $J_{ij}=0$. Consequently, the $(i,j)$ minor of $\Sigma$, denoted by $M_{ij}$, is zero for every $t\notin \mathcal{N}_i\cup\{i\}$. $M_{ij}$ is the determinant of the matrix that results from deleting row $i$ and column $j$ of $\Sigma$. This implies that for every $j\notin \mathcal{N}_i\cup\{i\}$, after removing row $i$ from $\Sigma$, the $i$-th column of the resulting matrix (denote it by $\Sigma'$) can be written as a linear combination of columns $\{1,...,n\}\setminus\{i,j\}$. Using this observation and the fact that $\Sigma$ is positive semidefinite, we obtain 
\vspace{-2mm}
\begin{equation}\label{use1}
\exists\ v\in \mathbb{R}^{|\mathcal{N}_i|};\ \ \Sigma'_{i}=\Sigma'_{\mathcal{N}_i}v,
\end{equation}
where $|\mathcal{N}_i|$ is the number of $i$'s neighbors. $\Sigma'_i$  and $\Sigma'_{\mathcal{N}_i}$ denote the $i$-th column of $\Sigma'$ and the sub-matrix of $\Sigma'$ comprising rows with the index set $\mathcal{N}_i$, respectively. Therefore, if $S=\{\mathcal{N}_i, S\setminus \mathcal{N}_i\}$ and $i\notin S$, Equation (\ref{use1}) implies 
\begin{equation}\nonumber
\Sigma_{S,i}=\Sigma_{S,S}\tiny{(\sqrt{D_S})^{-1}\sqrt{D_S}}\begin{bmatrix}
      v          \\[0.3em]
       0           
     \end{bmatrix}.
\end{equation}
Note that $\Sigma_{S,S}$ is non-singular, because it is a principal minors of a positive-semidefinite matrix $\Sigma$.
\vspace{-3mm}

\subsection{Proof of Proposition \ref{one}.}\label{one_proof}
\vspace{-1mm}
Denote the node of interest by $X_i$, and its neighbor by $X_j$. We show that
$
I(X_i;X_j)=\max_{s}I(X_i;X_s).
$
This is because, if we consider an arbitrary non-neighboring node $s$, then
\begin{align}\nonumber
\small{I(X_i;X_s,X_j,X_R)=I(X_i;X_j)+I(X_i;X_s,X_R|X_j)=I(X_i;X_s)+I(X_i;X_j,X_R|X_s)},
\end{align}
where $R=\{1,..,n\}\setminus\{i,j,s\}$. 
By the definition of RMF, $I(X_i;X_s,X_R|X_j)=0$, while $I(X_i;X_j,X_R|X_s)>0$. 

\vspace{-3mm}

\subsection{Proof of Lemma \ref{backi}.}\label{backi_proof}
\vspace{-1mm}

 Let $\epsilon_{F}=\frac{1}{2}\log\frac{1}{1-\epsilon}$, then
$
I(X_i;X_j|X_{S_i^{(k-1)}})\geq\epsilon_F
$
implies $\frac{\mathbb{E}^2[Y_i Y_j]}{\mathbb{E}[X_j^2]}\geq\epsilon\frac{\mathbb{E}[Y^2_{j}]\mathbb{E}[Y^2_{i}]}{\mathbb{E}[X_j^2]}:=\epsilon k_{i,j}$, where the latter is the amount of decrements in the loss function (\ref{loss}) after adding $t$ to the active set of node $i$. It is not hard to see that $\frac{1}{2}\log k_{i,j}=\frac{1}{2}\log \Sigma_{ii}-I(X_i;X_{S_i^{(k-1)}})-I(X_j;X_{S_i^{(k-1)}})$.

\vspace{-3mm}

\subsection{Proof of Lemma \ref{comp}.}\label{comp_proof}
\vspace{-1mm}
The first part is a direct consequence of Theorem \ref{prune}.
Next, we show that the $j^*$ obtained in the backward step of Algorithm 2 in \cite{johnson2011high} belongs to the set $L$ defined in the 12th line of Algorithm \ref{algorithm} given the active set $S_i^{(k-1)}$. Recall $\small{Y_i=X_i-(\beta')^TX_{S_i^{(k-1)}}}$, where $\beta'=\argmin_{\beta}\mathbb{E}[(X_i-\beta^T X_{S_i^{(k-1)}})^2]$.
Then,
\begin{align}\nonumber
\small{j^*=\argmin_{t\in S_i^{(k-1)}}\mathbb{E}\left[\left(Y_i+\beta'_{j}X_j\right)^2-\left(Y_i\right)^2\right]}
\small{=\argmin_{j\in S_i^{(k-1)}}\mathbb{E}\left[\beta'_{j}X_j\left(2Y_i+\beta'_j X_j\right)\right]}.\notag
\end{align}
By projection theorem, we must have $\mathbb{E}[X_j Y_i]=0$. Hence, $j^*=\argmin_{j\in S_i^{(k-1)}}\left(\beta'_{j}\right)^2\Sigma_{jj}$, which corresponds to the minimum entry of the vector $u^*$ defined in the 11-th line of Algorithm \ref{algorithm}.

\subsection{Proof of Theorem \ref{thm:SiUB}}\label{thm:SiUBp}
Assume that the algorithm was executed $m$ rounds before terminating. Denote the selected new neighbors at round $k$ by $\tilde{S}_i^{(k)}$, and the entire set of selected neighbors at the end of round $k$ by $S_i^{(k)}$. Let $S_i^{(0)}=\emptyset$. Then,
\begin{align}
\label{eq:condcovjump}
\Sigma_{ii}-\Sigma_{ii|S_i}&=\sum_{k=1}^{m}\left(\Sigma_{ii|S_i^{(k-1)}}-\Sigma_{ii|S_i^{(k)}}\right)\notag\\&=\sum_{k=1}^{m}\Sigma_{i,\tilde{S}_i^{(k)}|S_i^{(k-1)}}\Sigma^{-1}_{\tilde{S}_i^{(k)}\tilde{S}_i^{(k)}|S_i^{(k-1)}}\Sigma^T_{i,\tilde{S}_i^{(k)}|S_i^{(k-1)}}.
\end{align}
By Corollary \ref{coro:re}, the eigenvalues of $\Sigma_{\tilde{S}_i^{(k)}\tilde{S}_i^{(k)}|S_i^{(k-1)}}^{-1}$ are inside the region $[(1-\alpha)d_{min},(1+\alpha)d_{max}]$, and by the way the threshold is designed,
\begin{align}
\Vert\Sigma_{i,\tilde{S}_i^{(k)}|S_i^{(k-1)}}\Vert_2^2\geq |\tilde{S}_i^{(k)}|\tau^2.
\end{align}
Notice that the left hand side of (\ref{eq:condcovjump}) can be upper bounded by $\Sigma_{ii}-\Sigma_{ii|\mathcal{N}_i}$. To prove this, first notice that when $S_i^{(k)}\subseteq S_{i}^{(k+1)}$, we have $\Sigma_{ii|S_i^{(k+1)}}\leq\Sigma_{ii|S_i^{(k)}}$, as can be seen from the each individual term inside the summation of (\ref{eq:condcovjump}), since $\Sigma^{-1}_{\tilde{S}_i^{(k)}\tilde{S}_i^{(k)}|S_i^{(k-1)}}$ is positive definite. Secondly, when $\mathcal{N}_i\subseteq S_i^{(k)}$, $\Sigma_{i,\tilde{S}_i^{(k+1)}|S_i^{(k)}}=0$, by local Markov property. Combining these two observations shows that $\Sigma_{ii|S_i}$ is always non-increasing when new neighbors are selected, but remains unchanged once all neighbors has been selected, which proves the upper bound for the left hand side.

Hence,
\begin{align}
(1-\alpha)d_{min}\tau^2|S_i|&\leq\Sigma_{ii}-\Sigma_{ii|\mathcal{N}_i}=\Sigma_{i,\mathcal{N}_i}\Sigma_{\mathcal{N}_i,\mathcal{N}_i}^{-1}\Sigma_{\mathcal{N}_i,i}=J_{i,\mathcal{N}_i}\Sigma_{\mathcal{N}_i,\mathcal{N}_i}J_{\mathcal{N}_i,i}\notag\\&\leq (1-\alpha)^{-1}d_{min}^{-1}b^2\Delta_i.
\end{align}
We thus arrive at the conclusion that
\begin{align}
|S_i|\leq\frac{b^2}{(1-\alpha)^2d_{min}^2\tau^2}\Delta_i.
\end{align}

\subsection{Proof of Proposition \ref{prop:relationship}}\label{prop:relationshipp}

Without loss of generality, assume node 1 has $\Delta$ neighbors, from node 2 to $\Delta+1$. Then, consider the subgraph involving these $\Delta+1$, whose partial correlation matrix is denoted by $\tilde{R}$. Denote the adjacency matrix of a star graph with $\Delta+1$ nodes by $A$, then, since no triangles exist in the graph, we have
\begin{align}
\frac{a}{d_{max}}\Vert A\Vert_2\leq\Vert |\tilde{R}|\Vert\leq\alpha, 
\end{align}
where the first step is because $\rho(A)\leq \rho(B)$ when $A\leq B$ entry-wise and when $A$ is positive, and the second inequality is by assumption.

Notice that for a star graph of dimension $\Delta+1$, the spectral radius of adjacency matrix can be easily computed (by definition of eigenvalue) to be $\sqrt{\Delta}$. Hence,
\begin{align}
\Delta\leq\left(\frac{d_{max}\alpha}{a}\right)^2,
\end{align}
where the equality hold when the subgraph containing any one of the nodes with degree $\Delta$ and all of its neighbors is a star graph, with $d_{min}=d_{max}$ and all non-zero off-diagonal entries take the same value $a$.

\subsection{Proof of Lemma \ref{lm:normalizedlb} }\label{lm:normalizedlbp}

We start off by proving the normalized case and then extend the proof to the generalized case. For the normalized case, the diagonal elements of $J$ are ones, as given in Example \ref{eg:normalized}. From (\ref{eq:conditioning}), we know that the conditional covariance matrix can be obtained by inverting $J_{SS}$. Notice that $J_{S,S}$ preserves the structure of the original graph on the subset of nodes $S$. Hence, we can always treat $J_{S,S}$ as the $J$ matrix of the subgraph defined on set of nodes $S$. Hence it's sufficient to show
\begin{align}
\max_{j\in\mathcal{N}_i}\Sigma^2_{ij}\geq\frac{1}{K}\frac{\Vert J_{i,\mathcal{N}_i}\Vert_2^2}{(1-\Vert J_{i,\mathcal{N}_i}\Vert_2^2)^2}.
\end{align}
Notice that when $J$ is normalized, we have (\ref{eq:neumman}), which implies $\Sigma_{i,\mathcal{N}_i}=-J_{i,\mathcal{N}_i}\Sigma_{\mathcal{N}_i,\mathcal{N}_i}$. Thus, we have
\begin{align}
\sum_{j\in\mathcal{N}_i}\Sigma_{ij}^2&=J_{i,\mathcal{N}_i}\Sigma_{\mathcal{N}_i,\mathcal{N}_i}\Sigma_{\mathcal{N}_i,\mathcal{N}_i}^TJ_{i,\mathcal{N}_i}^T.
\end{align}

We now claim that
\begin{align}
J_{i,\mathcal{N}_i}\Sigma_{\mathcal{N}_i,\mathcal{N}_i}\Sigma_{\mathcal{N}_i,\mathcal{N}_i}^TJ_{i,\mathcal{N}_i}^T=\frac{J_{i,\mathcal{N}_i}\Lambda\Lambda^T J_{i,\mathcal{N}_i}^T}{(1-J_{i,\mathcal{N}_i}\Lambda J_{i,\mathcal{N}_i}^T)^2}
\end{align}
where $\Lambda=(J_{\mathcal{N}_i,\mathcal{N}_i}-J_{\bar{\mathcal{N}}_i,\mathcal{N}_i}^TJ_{\bar{\mathcal{N}}_i,\bar{\mathcal{N}}_i}^{-1}J_{\bar{\mathcal{N}}_i,\mathcal{N}_i})^{-1}$, and $\bar{\mathcal{N}}_i:=V\backslash(\mathcal{N}_i\cup \{i\})$. To show this, we exploit the fact that $J_{ij}=0$ when $i$ and $j$ are non-neighbors.

Let $S=\mathcal{N}_i\cup i$. By block matrix inverse,
\begin{align}
\Sigma_{S,S}=\left(J_{S,S}-J_{S,S^c}J_{S^c,S^c}^{-1}J_{S^c,S}\right)^{-1}.
\end{align}
Notice that the first row of $J_{S,S^c}$ and the first column of $J_{S^c,S}$ are all zeros, we must have
\begin{align}
J_{S,S}-J_{S,S^c}J_{S^c,S^c}^{-1}J_{S^c,S}=\left[\begin{array}{cc}J_{ii} & J_{i,\mathcal{N}_i} \\ J_{\mathcal{N}_i,i} & \Lambda^{-1}\end{array}\right],
\label{eq:temp}
\end{align}
where $\Lambda$ is the block matrix introduced previously.

Remember that we wish to find $\Sigma_{\mathcal{N}_i,\mathcal{N}_i}$, which is the block matrix at the bottom right corner when inverting the right hand side of (\ref{eq:temp}). Hence, by Schur's complement, we have
\begin{align}
\Sigma_{\mathcal{N}_i,\mathcal{N}_i}&=\Lambda+\Lambda J_{\mathcal{N}_i,i}(1-J_{i,\mathcal{N}_i}\Lambda J_{\mathcal{N}_i,i})^{-1}J_{i,\mathcal{N}_i}\Lambda\notag\\&=\Lambda+\frac{\Lambda J_{\mathcal{N}_i,i}J_{i,\mathcal{N}_i}\Lambda}{1-J_{i,\mathcal{N}_i}\Lambda J_{\mathcal{N}_i,i}}.
\end{align}
Hence,
\begin{align}
J_{i,\mathcal{N}_i}\Sigma_{\mathcal{N}_i,\mathcal{N}_i}&=J_{i,\mathcal{N}_i}\Lambda+\frac{J_{i,\mathcal{N}_i}\Lambda J_{\mathcal{N}_i,i}J_{i,\mathcal{N}_i}\Lambda}{1-J_{i,\mathcal{N}_i}\Lambda J_{\mathcal{N}_i,i}}\notag\\&=\left(1+\frac{J_{i,\mathcal{N}_i}\Lambda J_{\mathcal{N}_i,i}}{1-J_{i,\mathcal{N}_i}\Lambda J_{\mathcal{N}_i,i}}\right)J_{i,\mathcal{N}_i}\Lambda\notag\\&=\frac{J_{i,\mathcal{N}_i}\Lambda}{1-J_{i,\mathcal{N}_i}\Lambda J_{\mathcal{N}_i,i}}.
\end{align}
Hence the claim holds, and we have
\begin{align}
\label{eq:keystep}
\sum_{j\in\mathcal{N}_i}\Sigma_{ij}^2=\frac{J_{i,\mathcal{N}_i}\Lambda\Lambda^T J_{i,\mathcal{N}_i}^T}{(1-J_{i,\mathcal{N}_i}\Lambda J_{i,\mathcal{N}_i}^T)^2}.
\end{align}

Notice that the eigenvalue of $\Lambda$ is bounded within $[(1+\alpha)^{-1},(1-\alpha)^{-1}]$, and that the quadratic forms in both numerator and denominator share the same eigenvectors. Hence, the quadratic forms in both numerator and denominator achieve the minimum at the same time, and by the positive definiteness of the right hand side of (\ref{eq:temp}), we have $J_{i,\mathcal{N}_i}\Lambda J_{\mathcal{N}_i,i}<1$. Hence,
\begin{align}
\sum_{j\in\mathcal{N}_i}\Sigma_{ij}^2\geq\frac{\Vert J_{i,\mathcal{N}_i}\Vert_2^2}{((1+\alpha)-\Vert J_{i,\mathcal{N}_i}\Vert_2^2)^2},
\end{align}

The proof is completed by noting that the maximum of a group of real values is lower bounded by the average.

For the generalized case, $J=\sqrt{D}J_{norm}\sqrt{D}$, which indicates that 
\begin{align}
J_{i,\mathcal{N}_i}=\sqrt{d_{ii}}J_{norm,i,\mathcal{N}_i}\sqrt{D_{\mathcal{N}_i,\mathcal{N}_i}},
\end{align}
and
\begin{align}
\sqrt{d_{ii}}\Sigma_{i,\mathcal{N}_i}=-J_{norm,i,\mathcal{N}_i}\sqrt{D_{\mathcal{N}_i,\mathcal{N}_i}}\Sigma_{\mathcal{N}_i,\mathcal{N}_i}.
\end{align}
Hence,
\begin{align}
\sum_{j\in\mathcal{N}_i}\Sigma_{ij}^2&=\frac{J_{norm,i,\mathcal{N}_i}\sqrt{D_{\mathcal{N}_i,\mathcal{N}_i}}\Sigma_{\mathcal{N}_i,\mathcal{N}_i}\Sigma^T_{\mathcal{N}_i,\mathcal{N}_i}\sqrt{D_{\mathcal{N}_i,\mathcal{N}_i}}^TJ^T_{norm,i,\mathcal{N}_i}}{d_{ii}}\notag\\&=\frac{J_{i,\mathcal{N}_i}\Sigma_{\mathcal{N}_i,\mathcal{N}_i}\Sigma^T_{\mathcal{N}_i,\mathcal{N}_i}J^T_{i,\mathcal{N}_i}}{d_{ii}^2}=\frac{1}{d^2_{ii}}\frac{J_{i,\mathcal{N}_i}\Lambda\Lambda^TJ^T_{i,\mathcal{N}_i}}{(d_{ii}-J_{i,\mathcal{N}_i}\Lambda J^T_{i,\mathcal{N}_i})^2}\notag\\&\geq \frac{\Vert J_{i,\mathcal{N}_i}\Vert_2^2}{d_{ii}^2(d_{ii}(1+\alpha)d_{max}-\Vert J_{i,\mathcal{N}_i}\Vert_2^2)^2},
\end{align}
where $\Lambda=(J_{\mathcal{N}_i,\mathcal{N}_i}-J_{\bar{\mathcal{N}}_i,\mathcal{N}_i}^TJ_{\bar{\mathcal{N}}_i,\bar{\mathcal{N}}_i}^{-1}J_{\bar{\mathcal{N}}_i,\mathcal{N}_i})^{-1}$.

\subsection{Proof of Corollary \ref{cor:trianglefree}}
\label{cor:trianglefreep}

We first prove a corresponding version of Lemma \ref{lm:normalizedlb} when the graph does not contain any triangles.

Under Assumption \ref{asm:asm}, denote the estimated neighborhood of node $i$ at any point by $S_i$, and assume the graph is free from triangles, and that there are $K$ neighbors of node $i$ undiscovered. Then
\begin{align}
\max_{j\in\mathcal{N}_i\backslash S_i}\Sigma_{ij|S_i}^2\geq\frac{1}{K}\frac{\Vert J_{i,\mathcal{N}_i\backslash S_i}\Vert_2^2}{d_{ii}^2(d_{ii}\max_{j\in\mathcal{N}_i\backslash S_i}d_{jj}-\Vert J_{i,\mathcal{N}_i\backslash S_i}\Vert_2^2)^2}.
\end{align}

To prove this, we again consider the normalized case. The first few steps are identical, we hence start from equation (\ref{eq:keystep}).

Since the graph does not contain any triangles, and $J$ is normalized with diagonal elements being 1, $J_{\mathcal{N}_i\mathcal{N}_i}$ must be an identity matrix. Notice that the following normalized symmetric positive definite matrix
\begin{align}
\left[\begin{array}{cc} I & B\\B^T & C\end{array}\right],
\end{align}
satisfies the property that the largest eigenvalue of $BC^{-1}B^T$ is 1, and $BC^{-1}B^T$ is positive definite (since $I-BC^{-1}B^T$ has to be positive definite and $C^{-1}$ is positive definite). Since $\Lambda^{-1}$ is of the above form, the minimum eigenvalue of $\Lambda$ is 1. Noticing that $J_{i,\mathcal{N}_i}\Lambda J_{i,\mathcal{N}_i}^T<1$ (so that the numerator without squaring is always positive), and that $\Lambda^T\Lambda$ shares the same eigenvectors with $\Lambda$, we have
\begin{align}
\sum_{j\in\mathcal{N}_i}\Sigma_{ij}^2\geq\frac{\Vert J_{i,\mathcal{N}_i}\Vert^2_2}{(1-\Vert J_{i,\mathcal{N}_i}\Vert^2_2)^2}.
\end{align}
The proof is completed by noting that the maximum of a group of real values is lower bounded by the average.

For the generalized case, we have, by similar argument,
\begin{align}
\sum_{j\in\mathcal{N}_i}\Sigma_{ij}^2&=\frac{1}{d^2_{ii}}\frac{J_{i,\mathcal{N}_i}\Lambda\Lambda^TJ^T_{i,\mathcal{N}_i}}{(d_{ii}-J_{i,\mathcal{N}_i}\Lambda J^T_{i,\mathcal{N}_i})^2}\notag\\&\geq \frac{\Vert J_{i,\mathcal{N}_i}\Vert_2^2}{d_{ii}^2(d_{ii}\max_{j\in\mathcal{N}_i} d_{jj}-\Vert J_{i,\mathcal{N}_i}\Vert_2^2)^2},
\end{align}
where $\Lambda=(J_{\mathcal{N}_i,\mathcal{N}_i}-J_{\bar{\mathcal{N}}_i,\mathcal{N}_i}^TJ_{\bar{\mathcal{N}}_i,\bar{\mathcal{N}}_i}^{-1}J_{\bar{\mathcal{N}}_i,\mathcal{N}_i})^{-1}$.

The proof of the corollary is then completed by observing that $\Vert J_{i,\mathcal{N}_i\backslash S_i}\Vert_2^2\geq Ka^2$, and $K\geq 1$.

\subsection{Proof of Theorem \ref{thm:sparsistency}}\label{thm:sparsistencyp}

To perform sample based analysis, we first cite the results in \cite{ravikumar2011high} and \cite{anandkumar2012high}, which provide concentration guarantees of the covariance and conditional covariance. For our convenience, we translate the notation involved in those results.
\begin{lemma}[Concentration of empirical covariances \cite{ravikumar2011high}\cite{anandkumar2012high}]
\label{lm:covconcentration}
For any $n$ dimensional Gaussian random vector $X=(X_1,...,X_n)$, the empirical covariance obtained from $N$ i.i.d. samples satisfies
\begin{align}
\mathbbm{P}\left[\left\vert\hat{\Sigma}_{ij}-\Sigma_{ij}\right\vert>\varepsilon\right]\leq 4\exp\left[-\frac{N\varepsilon^2}{3200M^2}\right],
\end{align}
for all $\varepsilon\in(0,40M)$ and $M=\max_i\Sigma_{ii}$.
\end{lemma}
\begin{lemma}[Concentration of empirical conditional covariance \cite{anandkumar2012high}]\label{lm:condcovconcentration}
For a walk summable $n$-dimensional Gaussian graphical model where $X=(X_1,...,X_n)$, and
\begin{align}
\hat{\Sigma}_{ij|S}=\hat{\Sigma}_{ij}-\hat{\Sigma}_{i,S}\hat{\Sigma}_{S,S}^{-1}\hat{\Sigma}_{S,j},
\end{align}
we have
\begin{align}
\mathbbm{P}\left[\max_{\substack{i\neq j\\S\subset V,|S|\leq\eta}}\left\vert\hat{\Sigma}_{ij|S}-\Sigma_{ij|S}\right\vert>\varepsilon\right]\leq4n^{\eta+2}\exp\left[-\frac{N\varepsilon^2}{C_1}\right],
\end{align}
where $C_1\in(0,\infty)$ is a bounded constant if $\Vert\Sigma\Vert_{\infty}<\infty$, and again $\varepsilon\in(0,40M)$, and $N\geq\eta$.
\end{lemma}

From these two lemmas, we are able to show that the backward pruning algorithm succeeds with high probability as well, in the following lemma.
\begin{lemma}[Pruning correctness]\label{lm:pruning}
Assume that for node $i$, the estimated neighborhood $S_i$ chosen by Algorithm \ref{alg:thresholding} satisfies $\mathcal{N}_i\subseteq S_i$. Let $\Gamma=\Sigma_{i,S_i}\Sigma_{S_i,S_i}^{-1}$. Then
\begin{align}
\mathbbm{P}\left[\Vert\hat{\Gamma}-\Gamma\Vert_{\infty}>\varepsilon\right]\leq 4\exp\left[-\frac{NC_2|S_i|^6\varepsilon^2}{3200M^2}\right],
\end{align}
where $N$ is the number of samples. $C_2\in(0,\infty)$ is a bounded constant if $\Vert\Sigma\Vert_\infty<\infty$, $\varepsilon\in(0,40M)$, and $N\geq |S_i|$.
\end{lemma}
\begin{proof}
See Appendix \ref{lm:pruningp}
\end{proof}

We now prove the result of Theorem \ref{thm:sparsistency} with the help from the above lemmas.

By union bound,
\begin{align}
\mathbbm{P}[\text{Algorithms fail}]\leq\mathbbm{P}[\text{Algorithm 1 fails}]+\mathbbm{P}[\text{Algorithm 2 fails} | \text{Algorithm 1 succeeds}].
\end{align}
For Algorithm \ref{alg:thresholding}, consider directly Lemma \ref{lm:condcovconcentration}. The algorithm performs at most $\Delta$ rounds, which is upper bounded by $d_{min}\alpha/a$. In each round, the algorithm checks the conditional covariance over all unselected nodes, and if one has a larger discrepancy between $\hat{\Sigma}_{ij|S_i}$ and $\Sigma_{ij|S_i}$ then $\epsilon$, then an error occurs. Thus, it can be obtained by union bound, that
\begin{align}
\mathbbm{P}[\text{Algorithms fail}]&\leq \frac{nd_{min}\alpha}{a}\mathbbm{P}\left[\max_{\substack{i\neq j\\S\subset V,|S|\leq\eta}}\left\vert\hat{\Sigma}_{ij|S}-\Sigma_{ij|S}\right\vert>\epsilon\right]\notag\\&\leq\frac{4d_{min}\alpha}{a}n^{\eta+3}\exp\left[-\frac{N\epsilon^2}{C_1}\right],
\end{align}
where
\begin{align}
\eta=\frac{1}{(1-\alpha)^2d_{min}^2(ad_{max}^{-1}(d_{max}^2-a^2)^{-1/2}-\epsilon)^2}
\end{align}
is the upper bound for $|S_i|$.

Since the upper bound does not depend on the graph size, there exist a constant $C_{3}=\eta+3+C_{31}$, for which if $N=C_{3}\log n$, then $\mathbbm{P}[\text{Algorithm 1 fails}]\leq C_{41}\exp(-C_{51}N)$, where $C_{41}=4d_{min}\alpha/a$ and $C_{51}=C_{31}/C_{3}$.

Next, consider the backward pruning algorithm. The algorithm fails if one of the non-neighbors of node $i$ has corresponding entry in $\hat{\Gamma}$ that's greater than $\epsilon=\tau^p$, or when the an actual neighbor of node $i$ has its corresponding entry in $\hat{\Gamma}$ smaller than $\tau^p$. Hence, by Lemma \ref{lm:pruning}, we have
\begin{align}
\mathbbm{P}[\text{Algorithm 2 fails} | \text{Algorithm 1 succeeds}]&\leq C_{42}\exp(-C_{52}N),
\end{align}
for any $N$, where $C_{42}=4$, and
\begin{align}
C_{52}=\min\left\{\frac{NC_2\epsilon^2}{3200M^2},\frac{NC_2(a-\epsilon)^2}{3200M^2}\right\}.
\end{align}
Hence, the theorem holds true by letting $C_3=\eta+3+C_{31}$, $C_4=\max\{C_{41},C_{42}\}$, and $C_5=\min\{C_{51},C_{52}\}$.

\subsection{Proof of Proposition \ref{prop:scale}}\label{prop:scalep}

The left part of the condition makes sure that the graph is not empty. Denote the adjacency matrix of the graph as $A$, then, we have
\begin{align}
\Vert |R|\Vert_2\leq \frac{b}{d_{min}}\Vert A\Vert\leq \frac{b}{d_{min}}\Delta<\alpha.
\end{align}
Since for any $|R|_{ij}\neq 0$, $|R|_{ij}\leq\frac{b}{d_{min}}$, the first inequality follows directly from Gelfand formula. The second step holds according to \cite{brualdi1985spectral}, which upper bounds the spectral radius of a graph's adjacency matrix by its degree upper bound.

Since there exists $J$ matrix of arbitrary dimension $n$ that satisfies the constraints imposed by the parameters $a,b,\alpha,\Delta$, where $\Delta<\frac{\alpha}{b}$, the size of the graph can thus be arbitrarily large for any set of parameters satisfying the specified condition in this proposition. In addition, the degree upper bound $\Delta$ can be arbitrarily large as well, if $a$ and $b$ are small.

\subsection{Proof of Lemma \ref{lm:pruning}}\label{lm:pruningp}

For any $S$ containing $m$ elements, assume $\hat{\Sigma}_{S,S}=\Sigma_{S,S}+F$, and $\hat{\Sigma}_{S,S}^{-1}=\Sigma_{S,S}^{-1}+E$. Then
\begin{align}
(\Sigma_{S,S}^{-1}+E)(\Sigma_{S,S}+F)=I,
\end{align}
indicating that
\begin{align}
E=-\Sigma_{S,S}^{-1}F(\Sigma_{S,S}+F)^{-1}.
\end{align}
Consider the matrix norm $\vertiii{ \Sigma_{S,S}}:=m\max_{i,j\in S}|\Sigma_{ij}|$ (page 342, \cite{horn2012matrix}). Then
\begin{align}
\vertiii{E}\leq\vertiii{\Sigma_{S,S}^{-1}}\vertiii{(\Sigma_{S,S}+F)^{-1}}\vertiii{F},
\end{align}
implying that for some constant $C$,
\begin{align}
\max_{1\leq i,j\leq m}|E_{ij}|\leq Cm^2\max_{1\leq i,j\leq m}|F_{ij}|.
\end{align}
Now let $S=S_i$. Combining with the boundedness of $\Sigma_{iS_i}$ and $\Sigma_{S_iS_i}^{-1}$, there exists a bounded constant $C_2$, such that when $|\hat{\Sigma}_{ij}-\Sigma_{ij}|<\varepsilon$ with high probability, $|\hat{\Sigma}_{iS_i}\hat{\Sigma}_{S_iS_i}^{-1}-\Sigma_{iS_i}\Sigma_{S_iS_i}^{-1}|<C_2m^3\varepsilon$ with high probability. Hence the result follows by applying Lemma \ref{lm:covconcentration}.

\end{document}